\documentclass[]{article}

\usepackage{amsmath, amssymb, amsthm, bbm, booktabs, color, enumitem, graphicx, multicol, url, xcolor, animate, rotating, multirow, lineno, booktabs, tikz, xurl}

\usepackage[labelfont=bf]{caption}
\usepackage[a4paper]{geometry}
\usepackage[authoryear]{natbib}
\usepackage[colorlinks,citecolor=blue]{hyperref}

\newcommand{\real}{\mathbb{R}}

\newcommand{\one}{\mathbbm{1}}
\newcommand{\hsp}{\hspace{0.2mm}}

\newcommand{\cl}{{\rm cl}} 
\newcommand{\rk}{{\rm rk}} 
\newcommand{\mrk}{\overline{\rk}} 

\newcommand{\cov}{{\rm cov}} 

\newcommand{\AUC}{{\rm AUC}} 
\newcommand{\CPA}{{\rm CPA}} 
\newcommand{\ROC}{{\rm ROC}} 

\newcommand{\cC}{{\cal C}} 

\newcommand{\tauK}{\tau_{\scriptstyle \hsp \rm K}}  
\newcommand{\rhoM}{\rho_{\scriptstyle \rm M}}  
\newcommand{\rhoS}{\rho_{\scriptstyle \rm S}}  

\newcommand{\Upairs}{U_{\scriptstyle \rm pairs}}  
\newcommand{\Uovo}{U_{\scriptstyle \rm ovo}}  
\newcommand{\Ucons}{U_{\scriptstyle \rm cons}}  

\newtheorem{theorem}{Theorem}
\newtheorem{definition}{Definition}

\title{Receiver operating characteristic (ROC) movies, universal ROC
	(UROC) curves, and coefficient of predictive ability
	(CPA)}
\author{Tilmann Gneiting$^{1,2}$ and Eva-Maria Walz$^{2,1}$
	\vspace{0.5cm}\\
$^{1}$Heidelberg Institute for Theoretical Studies (HITS), Germany \\
$^{2}$Institute for Stochastics, Karlsruhe Institute of Technology (KIT), Germany}

\begin{document}

\maketitle

\begin{abstract}
Throughout science and technology, receiver operating characteristic
(ROC) curves and associated area under the curve ($\AUC$) measures
constitute powerful tools for assessing the predictive abilities of
features, markers and tests in binary classification problems.
Despite its immense popularity, ROC analysis has been subject to a
fundamental restriction, in that it applies to dichotomous (yes or no)
outcomes only.  Here we introduce ROC movies and universal ROC (UROC)
curves that apply to just any linearly ordered outcome, along with an
associated coefficient of predictive ability ($\CPA$) measure.  $\CPA$
equals the area under the UROC curve, and admits appealing
interpretations in terms of probabilities and rank based covariances.
For binary outcomes $\CPA$ equals $\AUC$, and for pairwise distinct
outcomes $\CPA$ relates linearly to Spearman's coefficient, in the
same way that the C index relates linearly to Kendall's coefficient.
ROC movies, UROC curves, and $\CPA$ nest and generalize the tools of
classical ROC analysis, and are bound to supersede them in a wealth of
applications.  Their usage is illustrated in data examples from
biomedicine and meteorology, where rank based measures yield new
insights in the WeatherBench comparison of the predictive performance
of convolutional neural networks and physical-numerical models for
weather prediction.
\vspace{0.5cm}\\
\textit{Keywords}: C index; classification and regression; evaluation
	metric; rank correlation; coefficient; ROC analysis
\end{abstract}

\section{Introduction}  \label{sec:introduction}

Originating from signal processing and psychology, popularized in the
1980s \citep{Hanley1982, Swets1988}, and witnessing a surge of usage
in machine learning \citep{Bradley1997, Huang2005, Fawcett2006,
	Flach2016}, receiver operating characteristic or relative operating
characteristic (ROC) curves and area under the ROC curve ($\AUC$)
measures belong to the most widely used quantitative tools in science
and technology.  Strikingly, a Web of Science topic search for the
terms ``receiver operating characteristic'' or ``ROC'' yields well
over 15,000 scientific papers published in calendar year 2019 alone.
In a nutshell, the ROC curve quantifies the potential value of a
real-valued classifier score, feature, marker, or test as a predictor
of a binary outcome.  To give a classical example,
Fig.~\ref{fig:ROCC_survival} illustrates the initial levels of two
biomedical markers, serum albumin and serum bilirubin, in a Mayo
Clinic trial on primary biliary cirrhosis (PBC), a chronic fatal
disease of the liver \citep{Dickson1989}.  While patient records
specify the duration of survival in days, traditional ROC analysis
mandates the reduction of the outcome to a binary event, which here we
take as survival beyond four years.  Assuming that higher marker
values are more indicative of survival, we can take any threshold
value to predict survival if the marker exceeds the threshold, and
non-survival otherwise.  This type of binary classifier yields true
positives, false positives (erroneous predictions of survival), true
negatives, and false negatives (erroneous predictions of
non-survival).  The ROC curve is the piecewise linear curve that plots
the true positive rate, or sensitivity, versus the false positive
rate, or one minus the specificity, as the threshold for the
classifier moves through all possible values.

Despite its popularity, ROC analysis has been subject to a fundamental
shortcoming, namely, the restriction to binary outcomes.  Real-valued
outcomes are ubiquitous in scientific practice, and investigators have
been forced to artificially make them binary if the tools of ROC
analysis are to be applied.  In this light, researchers have been
seeking generalizations of ROC analysis that apply to just any type of
ordinal or real-valued outcomes in natural ways \citep{Etzioni1999,
  Heagerty2000, Bi2003, Pencina2004, Heagerty2005, Rosset2005,
  Mason2009, Orallo2013}.  Still, notwithstanding decades of
scientific endeavor, a fully satisfactory generalization has been
elusive.

\begin{figure}[t]
\centering
\includegraphics[width = \textwidth]{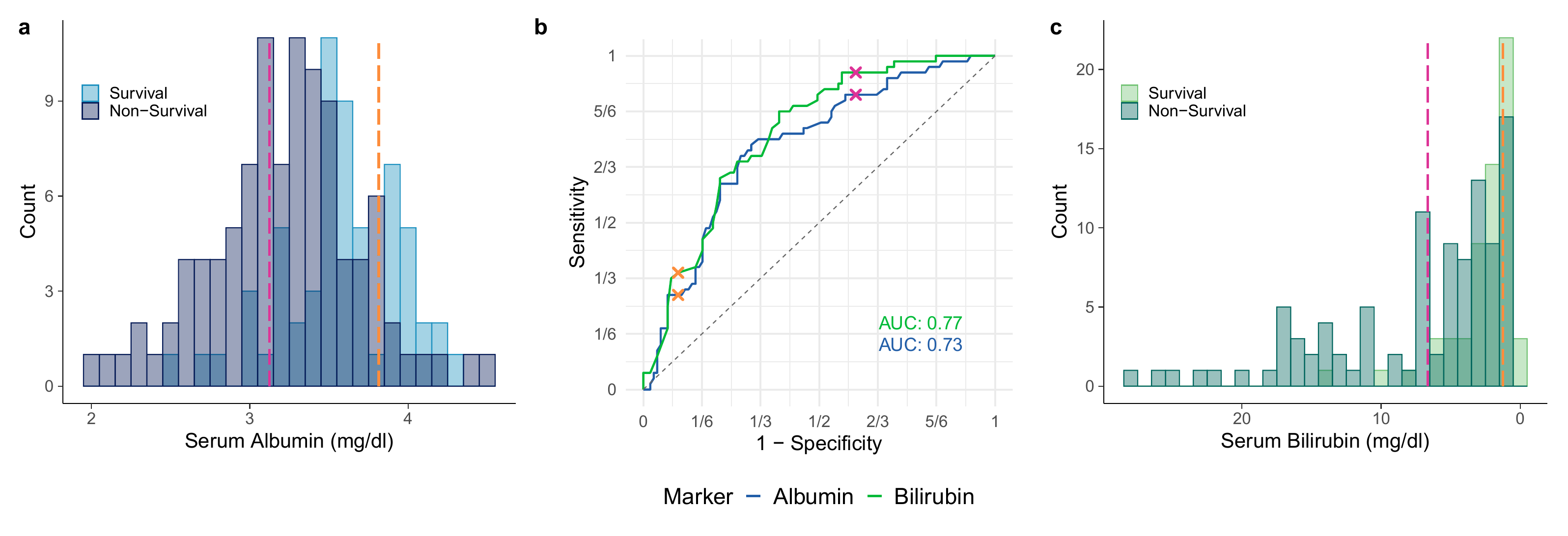}
\caption{Traditional ROC curves for two biomedical markers, serum
  albumin and serum bilirubin, as predictors of patient survival
  beyond a threshold value of 1462 days (four years) in a Mayo Clinic
  trial.  (a, c) Bar plots of marker levels conditional on survival or
  non-survival.  The stronger shading results from overlap.  For
  bilirubin, we reverse orientation, as is customary in the biomedical
  literature.  (b) ROC curves and $\AUC$ values.  The crosses
  correspond to binary classifiers at the feature thresholds indicated
  in the bar plots.  \label{fig:ROCC_survival}}
\end{figure}

In this paper, we propose a powerful generalization of ROC analysis,
which overcomes extant shortcomings, and introduce data science tools
in the form of the ROC movie, the universal ROC (UROC) curve, and an
associated, rank based coefficient of (potential) predictive ability
($\CPA$) measure — tools that apply to just any linearly ordered
outcome, including both binary, ordinal, mixed discrete-continuous,
and continuous variables.  The ROC movie comprises the sequence of the
traditional, static ROC curves as the linearly ordered outcome is
converted to a binary variable at successively higher thresholds.  The
UROC curve is a weighted average of the individual ROC curves that
constitute the ROC movie, with weights that depend on the class
configuration, as induced by the unique values of the outcome, in
judiciously predicated, well-defined ways.  $\CPA$ is a weighted
average of the individual $\AUC$ values in the very same way that the
UROC curve is a weighted average of the individual ROC curves that
constitute the ROC movie.  Hence, $\CPA$ equals the area under the
UROC curve.  This set of generalized tools reduces to the standard ROC
curve and $\AUC$ when applied to binary outcomes.  Moreover, key
properties and relations from conventional ROC theory extend to ROC
movies, UROC curves, and $\CPA$ in meaningful ways, to result in a
coherent toolbox that properly extends the standard ROC concept.  For
a graphical preview, we return to the survival data example from
Fig.~\ref{fig:ROCC_survival}, where the outcome was artificially made
binary.  Equipped with the new set of tools we no longer need to
transform survival time into a specific dichotomous outcome.  Figure
\ref{fig:ROCM_survival} shows the ROC movie, the UROC curve, and
$\CPA$ for the survival dataset.

The remainder of the paper is organized as follows.  Section
\ref{sec:binary} provides a brief review of conventional ROC analysis
for dichotomous outcomes.  The key technical development is in
Sections \ref{sec:ROC_movies} and \ref{sec:CPA}, where we introduce
and study ROC movies, UROC curves, and the rank based $\CPA$ measure.
To illustrate practical usage and relevance, real data examples from
survival analysis and weather prediction are presented in Section
\ref{sec:examples}.  We monitor recent progress in numerical weather
prediction (NWP) and shed new light on a recent comparison of the
predictive abilities of convolutional neural networks (CNNs)
vs.~traditional NWP models.  The paper closes with a discussion in
Section \ref{sec:discussion}.

\begin{figure}[t]
\centering
\animategraphics[label = myAnim7, height = 70mm]{7}{figures/imgeps4_out}{}{}
\caption{ROC movies, UROC curves, and $\CPA$ for two biomedical
  markers, serum albumin and serum bilirubin, as predictors of patient
  survival (in days) in a Mayo Clinic trial.  The ROC movies show the
  traditional ROC curves for binary events that correspond to patient
  survival beyond successively higher thresholds.  The numbers at
  upper left show the current value of the threshold in days, at upper
  middle the respective relative weight, and at bottom right the
  $\AUC$ values.  The threshold value of 1462 days recovers the
  traditional ROC curves in Fig.~\ref{fig:ROCC_survival}.  The video
  ends in a static screen with the UROC curves and $\CPA$ values for
  the two markers.  \label{fig:ROCM_survival}}
\end{figure}

\section{Receiver operating characteristic (ROC) curves and area under the curve (AUC) for binary outcomes}  
\label{sec:binary} 

Before we introduce ROC movies, UROC curves, and $\CPA$, it is
essential that we establish notation and review the classical case of
ROC analysis for binary outcomes, as described in review articles and
monographs by \citet{Hanley1982}, \citet{Swets1988},
\citet{Bradley1997}, \citet{Pepe2003}, \citet{Fawcett2006}, and
\citet{Flach2016}, among others.

\subsection{Binary setting}  \label{sec:setting}

Throughout this section we consider bivariate data of the form 
\begin{equation}  \label{eq:binary} 
(x_1,y_1), \ldots, (x_n,y_n) \, \in \, \real \times \{ 0, 1 \}, 
\end{equation} 
where $x_i \in \real$ is a real-valued classifier score, feature, marker,
or covariate value, and $y_i \in \{ 0, 1 \}$ is a binary outcome, for
$i = 1, \ldots, n$.  Following the extant literature, we refer to $y =
1$ as the positive outcome and to $y = 0$ as the negative outcome, and
we assume that higher values of the feature are indicative of stronger
support for the positive outcome.  Throughout we assume that there is
at least one index $i \in \{ 1, \ldots, n \}$ with $y_i = 0$, and a
further index $j \in \{ 1, \ldots, n \}$ with $y_j = 1$.

\subsection{Receiver operating characteristic (ROC) curves}  \label{sec:ROCC} 

We can use any threshold value $x \in \real$ to obtain a hard
classifier, by predicting a positive outcome for a feature value $>
x$, and predicting a negative outcome for a feature value $\leq x$.
If we compare to the actual outcome, four possibilities arise.  True
positive and true negative cases correspond to correctly classified
instances from class 1 and class 0, respectively.  Similarly, false
positive and false negative cases are misclassified instances from
class 1 and class 0, respectively.

Considering the data \eqref{eq:binary}, we obtain the respective
\textit{true positive rate}, \textit{hit rate}\/ or
\textit{sensitivity}\/ (se),
\[
{\rm se}(x) = 
\frac{\frac{1}{n} \sum_{i=1}^n \one \{ x_i > x, \hsp y_i = 1 \}}{\frac{1}{n} \sum_{i=1}^n \one \{ \hsp y_i = 1 \}}, 
\] 
and the \textit{false negative rate}, \textit{false alarm rate}\/ or
one minus the \textit{specificity}~(sp),
\[
1 - {\rm sp}(x) = 
\frac{\frac{1}{n} \sum_{i=1}^n \one \{ x_i > x, \hsp y_i = 0 \}}{\frac{1}{n} \sum_{i=1}^n \one \{ \hsp y_i = 0 \}}, 
\]
at the threshold value $x \in \real$, where the indicator $\one \{ A
\}$ equals one if the event $A$ is true and zero otherwise.

Evidently, it suffices to consider threshold values $x$ equal to any
of the unique values of $x_1, \ldots, x_n$ or some $x_0 < x_1$.  For
every $x$ of this form, we obtain a point
\[
(1 - {\rm sp}(x), {\rm se}(x))
\]
in the unit square.  Linear interpolation of the respective discrete
point set results in a piecewise linear curve from $(0,0)$ to $(1,1)$
that is called the \textit{receiver operating characteristic}\/ (ROC)
\textit{curve}.  For a mathematically oriented, detailed discussion of
the construction see Section 2 of \citet{Gneiting2018}.

\subsection{Area under the curve (AUC)}  \label{sec:AUC} 

The area under the ROC curve is a widely used measure of the
predictive potential of a feature and generally referred to as the
\textit{area under the curve}\/ ($\AUC$).

In what follows, a well-known interpretation of $\AUC$ in terms of
probabilities will be useful. To this end, we define the function
\begin{equation}  \label{eq:s}
s(x,x') = \one \{ x < x' \} + \frac{1}{2} \one \{ x = x' \}, 
\end{equation}
where $x, x' \in \real$.  For subsequent use, note that if $x$ and
$x'$ are ranked within a list, and ties are resolved by assigning
equal ranks within tied groups, then $s(x,x') = s(\rk(x),\rk(x'))$,
where $\rk(x)$ and $\rk(x')$ are the ranks of $x$ and $x'$.

We now change notation and refer to the feature values in class $i \in
\{ 0, 1 \}$ as $x_{ik}$ for $k = 1, \ldots, n_i$, where $n_0 =
\sum_{i=1}^n \one \{ y_i = 0 \}$ and $n_1 = \sum_{i=1}^n \one \{ y_i =
1 \}$, respectively.  Thus, we have rewritten \eqref{eq:binary} as
\begin{equation}  \label{eq:binary2class} 
(x_{01},0), \ldots, (x_{0n_0},0), (x_{11},1), \ldots, (x_{1n_1},1) \, \in \, \real \times \{ 0, 1 \}.  
\end{equation} 
Using the new notation, Result 4.10 of \citet{Pepe2003} states that
\begin{equation}  \label{eq:MWU}
\AUC = \frac{1}{n_0 n_1} \, \sum_{i=1}^{n_0} \sum_{j = 1}^{n_1} s(x_{0i},x_{1j}).  
\end{equation}
In words, $\AUC$ equals the probability that under random sampling a
feature value from a positive instance is greater than a feature value
from a negative instance, with any ties resolved at random.  Expressed
differently, $\AUC$ equals the tie-adjusted probability of concordance
in feature--outcome pairs, where we define instances $(x,y) \in
\real^2$ and $(x',y') \in \real^2$ with $y \not= y'$ to be
\textit{concordant}\/ if either $x > x'$ and $y > y'$, or $x < x'$ and
$y < y'$.  Similarly, instances $(x,y)$ and $(x',y')$ with $y \not=
y'$ are \textit{discordant}\/ if either $x > x'$ and $y < y'$, or $x <
x'$ and $y > y'$.

Further investigation reveals a close connection to Somers' $D$, a
classical measure of ordinal association \citep{Somers1962}.  This
measure is defined as
\[
D = \frac{n_c - n_d}{n_0 n_1},
\]
where $n_0 n_1$ is the total number of pairs with distinct outcomes
that arise from the data in \eqref{eq:binary2class}, $n_c$ is the
number of concordant pairs, and $n_d$ is the number of discordant
pairs.  Finally, let $n_e$ be the number of pairs for which the
feature values are equal.  The relationship \eqref{eq:MWU} yields
\[
\AUC = \frac{n_c}{n_0 n_1} + \frac{1}{2} \frac{n_e}{n_0 n_1}, 
\]
and as $n_0 n_1 = n_c + n_d + n_e$, it follows that
\begin{equation}  \label{eq:SomersD}
\AUC = \frac{1}{2} \left( D + 1 \right)
\end{equation}
relates linearly to Somers' $D$.  

To give an example, suppose that the real-valued outcome $Y$ and the
features $X$, $X'$ and $X''$ are jointly Gaussian.  Specifically, we
assume that the joint distribution of $(Y, X, X', X'')$ is
multivariate normal with covariance matrix
\begin{align}  \label{eq:MVN}
\begin{pmatrix}
1 & 0.8 & 0.5 & 0.2 \\
0.8 & 1 & 0.8 & 0.5 \\
0.5 & 0.8 & 1 & 0.8 \\
0.2 & 0.5 & 0.8 & 1 \\
\end{pmatrix}
.
\end{align}
In order to apply classical ROC analysis, the real-valued outcome $Y$
needs to be converted to a binary variable, namely, an event of the
type $Y_\theta = \one \{ Y \geq \theta \}$ of $Y$ being greater than
or equal to a threshold value $\theta$.  Figure \ref{fig:ROCC} shows
ROC curves for the features $X$, $X'$ and $X''$ as a predictor of the
binary variable $Y_1$, based on a sample of size $n = 400$.  The
$\AUC$ values for $X$, $X'$ and $X''$ as a predictor of $Y_1$ are .91,
.72 and .61, respectively.

\begin{figure}[t]
\centering
\includegraphics[width = 0.625 \textwidth]{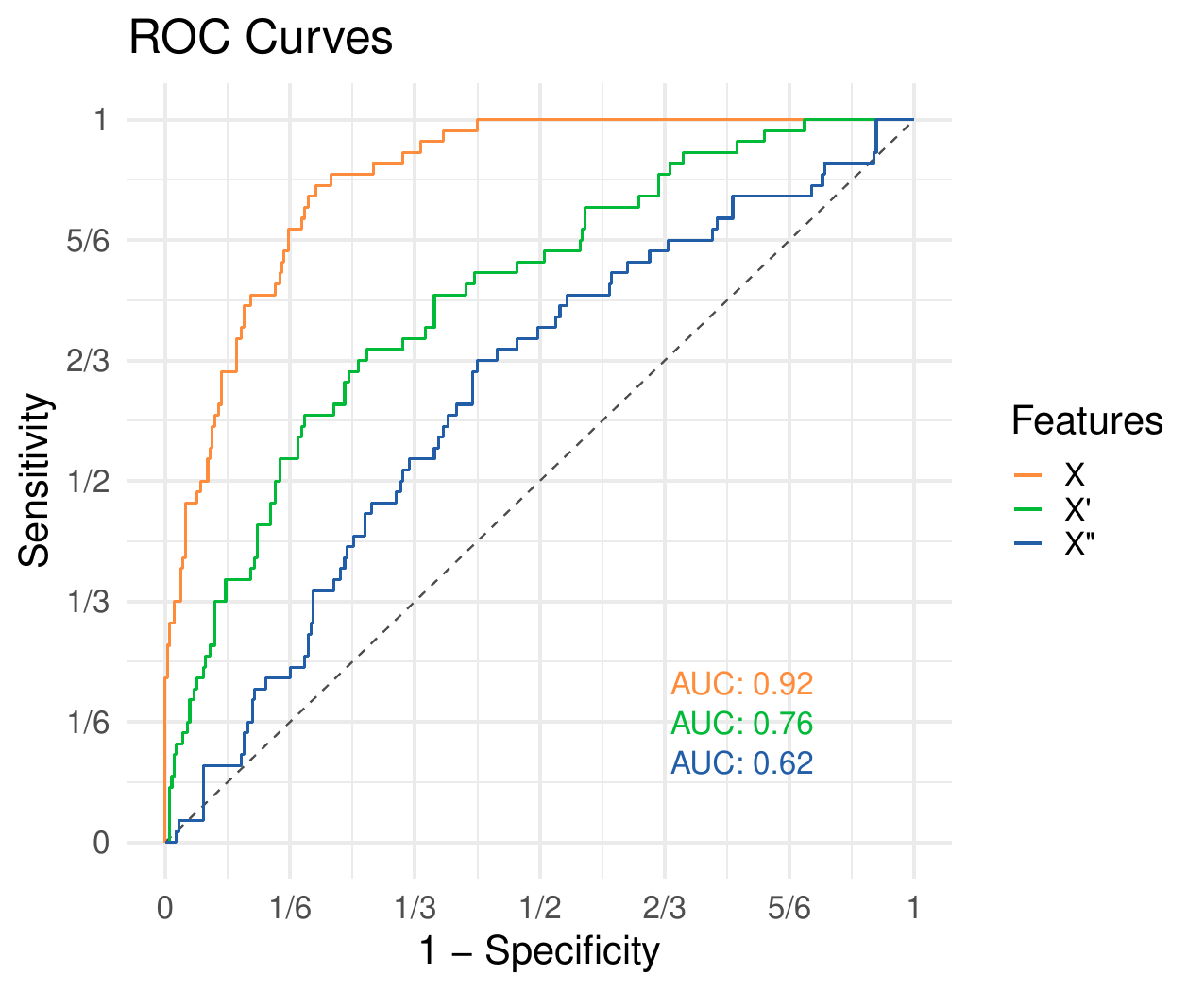}
\caption{Traditional ROC curves and $\AUC$ values for the features
  $X$, $X'$ and $X''$ as predictors of the binary outcome $Y_1 = \one
  \{ Y \geq 1 \}$ in the simulation example of Section \ref{sec:AUC},
  based on a sample of size $n = 400$.  \label{fig:ROCC}}
\end{figure}

\subsection{Key properties}  \label{sec:initialsummary}

A key requirement for a persuasive generalization of classical ROC
analysis is the reduction to ROC curves and $\AUC$ if the outcomes are
binary.  Furthermore, well established desirable properties from ROC
analysis ought to be retained.  To facilitate judging whether the
generalization in Sections \ref{sec:ROC_movies} and \ref{sec:CPA}
satisfies these desiderata, we summarize key properties of ROC curves
and $\AUC$ in the following (slightly informal) listing.
\begin{itemize} 
\item[(1)] The ROC curve and $\AUC$ are straightforward to compute and
  interpret, in the (rough) sense of \textit{the larger the better}.
\item[(2)] $\AUC$ attains values between 0 and 1 and relates linearly
  to Somers' $D$.  For a perfect feature, $\AUC = 1$ and $D = 1$; for
  a feature that is independent of the binary outcome, $\AUC =
  \frac{1}{2}$ und $D = 0$.
\item[(3)] The numerical value of $\AUC$ admits an interpretation as
  the probability of concordance for feature--outcome pairs.
\item[(4)] The ROC curve and $\AUC$ are purely rank based and,
  therefore, invariant under strictly increasing transformations.
  Specifically, if $\phi : \real \to \real$ is a strictly increasing
  function, then the ROC curve and $\AUC$ computed from
  \begin{equation}  \label{eq:binaryT}
  (\phi(x_1),y_1), \ldots, (\phi(x_n),y_n) \, \in \, \real \times \{ 0, 1 \}
  \end{equation} 
  are the same as the ROC curve and $\AUC$ computed from
  \eqref{eq:binary}.
\end{itemize} 
As an immediate consequence of the latter property, ROC curves and
$\AUC$ assess the discrimination ability or \textit{potential}\/
predictive ability of a classifier, feature, marker, or test
\citep{Wilks2019}.  Distinctly different methods are called for if one
seeks to evaluate a classifier's \textit{actual}\/ value in any given
applied setting \citep{Adams1999, Orallo2012, Ehm2016}.

\section{ROC movies and universal ROC (UROC) curves for real-valued outcomes} 
\label{sec:ROC_movies} 

As noted, traditional ROC analysis applies to binary outcomes only.
Thus, researchers working with real-valued outcomes, and desiring to
apply ROC analysis, need to convert and reduce to binary outcomes, by
thresholding artificially at a cut-off value.  Here we propose a
powerful generalization of ROC analysis, which overcomes extant
shortcomings, and introduce data analytic tools in the form of the ROC
movie, the universal ROC (UROC) curve, and an associated rank based
coefficient of (potential) predictive ability ($\CPA$) measure ---
tools that apply to just any linearly ordered outcome, including both
binary, ordinal, mixed discrete-continuous, and continuous variables.

\subsection{General real-valued setting}  \label{sec:real} 

Generalizing the binary setting in \eqref{eq:binary}, we now consider
bivariate data of the form
\begin{equation}  \label{eq:real} 
(x_1,y_1), \ldots, (x_n,y_n) \, \in \, \real \times \real, 
\end{equation} 
where $x_i$ is a real-valued point forecast, regression output,
feature, marker, or covariate value, and $y_i$ is a real-valued
outcome, for $i = 1, \ldots, n$.  Throughout we assume that there are
at least two unique values among the outcomes $y_1, \ldots, y_n$.

The crux of the subsequent development lies in a conversion to a
sequence of binary problems.  To this end, we let
\[
z_1 < \cdots < z_m 
\]
denote the $m \leq n$ unique values of $y_1, \ldots, y_n$, and
we define
\[
n_c = \sum_{i=1}^n \one \{ y_i = z_c \}
\]
as the number of instances among the outcomes $y_1, \ldots, y_n$ that
equal $z_c$, for $c = 1, \ldots, m$, so that $n_1 + \cdots + n_m = n$.
We refer to the respective groups of instances as \textit{classes}.

Next we transform the real-valued outcomes $y_1, \ldots, y_n$ into
binary outcomes $\one \{ y_1 \geq \theta \}, \ldots, \one \{ y_n \geq
\theta \}$ relative to a threshold value $\theta \in \real$.  Thus,
instead of analysing the original problem in \eqref{eq:real}, we
consider a series of binary problems.  By construction, only values of
$\theta$ equal to $z_2, \ldots, z_m$ result in nontrivial, unique sets
of binary outcomes.  Therefore, we consider $m - 1$ derived
classification problems with binary data of the form
\begin{equation}  \label{eq:real2binary}
(x_1, \one \{ y_1 \geq z_{c+1} \}), \ldots, (x_n, \one \{ y_n \geq z_{c+1} \}) 
\, \in \, \real \times \{ 0, 1 \}, 
\end{equation}
where $c = 1, \ldots, m - 1$.  As the derived problems are binary, all
the tools of traditional ROC analysis apply.

In the remainder of the section we describe our generalization of ROC
curves for binary data to ROC movies and universal ROC (UROC) curves
for real-valued data.  First, we argue that the $m - 1$ classical ROC
curves for the derived data in \eqref{eq:real2binary} can be merged
into a single dynamical display, to which we refer as a ROC movie
(Definition \ref{defn:ROCM}).  Then we define the UROC curve as a
judiciously weighted average of the classical ROC curves of which the
ROC movie is composed (Definition \ref{defn:UROC}).

Finally, we introduce a general measure of potential predictive
ability for features, termed the coefficient of predictive ability
(CPA).  $\CPA$ is a weighted average of the $\AUC$ values for the
derived binary problems in the very same way that the UROC curve is a
weighted average of the (classical) ROC curves that constitute the ROC
movie.  Hence, $\CPA$ equals the area under the UROC curve (Definition
\ref{defn:CPA}).  Alternatively, $\CPA$ can be interpreted as a
weighted probability of concordance (Theorem \ref{thm:CPA2L}) or in
terms of rank based covariances (Theorem \ref{thm:CPA2cov}).  $\CPA$
reduces to $\AUC$ if the outcomes are binary, and relates linearly to
Spearman's rank correlation coefficient if the outcomes are continuous
(Theorems \ref{thm:Spearman} and \ref{thm:Spearman.m}).

\subsection{ROC movies}  \label{sec:ROCM} 

We consider the sequence of $m - 1$ classification problems for the
derived binary data in \eqref{eq:real2binary}.  For $c = 1, \ldots, m
- 1$, we let $\ROC_c$ denote the associated ROC curve, and we let
$\AUC_c$ be the respective $\AUC$ value.

\begin{definition}  \label{defn:ROCM} 
For data of the form \eqref{eq:real}, the ROC \textit{movie}\/ is the
sequence $(\ROC_c)_{c = 1, \ldots, m-1}$ of the ROC curves for the
induced binary data in \eqref{eq:real2binary}.
\end{definition}

If the original problem is binary there are $m = 2$ classes only, and
the ROC movie reduces to the classical ROC curve.  In case the outcome
attains $m \geq 3$ distinct values the ROC movie can be visualized by
displaying the associated sequence of $m - 1$ ROC curves.  In medical
survival analysis, the outcomes $y_1, \ldots, y_n$ in data of the form
\eqref{eq:real} are survival times, and the analysis is frequently
hampered by censoring, as patients drop out of studies.  In this
setting, \citet{Etzioni1999} and \citet{Heagerty2000} introduced the
notion of time-dependent ROC curves, which are classical ROC curves
for the binary indicator $\one \{ y_i \geq t \}$ of survival through
(follow-up) time $t$, with censoring being handled efficiently.  For
an example see Fig.~2 of \cite{Heagerty2000}, where the ROC curves
concern survival beyond follow-up times of 40, 60, and 100 months,
respectively.  If the thresholds considered correspond to the unique
values of the outcomes, the sequence of time-dependent ROC curves
becomes a ROC movie in the sense of Definition \ref{defn:ROCM}, save
for the handling of censored data.  When the number $m \leq n$ of
classes is small or modest, the generation of the ROC movie is
straightforward.  Adaptations might be required as $m$ grows, and we
tend to this question in Section \ref{sec:NWP}.

\begin{figure}[t]
\centering
\animategraphics[label = myAnim7, height = 70mm]{28}{figures/imgeps3_reduced}{0}{}
\caption{ROC movies and UROC curves for the features $X$, $X'$ and
  $X''$ as predictors of the real-valued outcome $Y$ in the simulation
  example of Section \ref{sec:AUC}, based on the same sample as in
  Fig.~\ref{fig:ROCC}.  In the ROC movies, the number at upper left
  shows the threshold under consideration, the number at upper center
  the relative weight $w_c/\max_{l = 1, \ldots, m-1} w_l$ from
  \eqref{eq:w}, and the numbers at bottom right the respective $\AUC$
  values.  \label{fig:ROCM}}
\end{figure}

We have implemented ROC movies, UROC curves, and $\CPA$ within the
\texttt{uroc} package for the statistical programming language
\textsf{R} \citep{rcore2021} where the \texttt{animation} package of
\citet{Xie2013} provides functionality for converting \textsf{R}
images into a GIF animation, based on the external software
\texttt{ImageMagick}.  The \texttt{uroc} package can be downloaded
from \url{https://github.com/evwalz/uroc}.  In addition, a Python
\citep{python2021} implementation is available at
\url{https://github.com/evwalz/urocc}.  Returning to the example of
Section \ref{sec:AUC}, Fig.~\ref{fig:ROCM} compares the features $X$,
$X'$ and $X''$ as predictors of the real-valued outcome $Y$ in a joint
display of the three ROC movies and UROC curves, based on the same
sample of size $n = 400$ as in Fig.~\ref{fig:ROCC}.  In the ROC
movies, the threshold $z = 1.00$ recovers the traditional ROC curves
in Fig.~\ref{fig:ROCC}.

\subsection{Universal ROC (UROC) curves}  \label{sec:UROC} 

Next we propose a simple and efficient way of subsuming a ROC movie
for data of the form \eqref{eq:real} into a single, static graphical
display.  As before, let $z_1 < \cdots < z_m$ denote the unique values
of $y_1, \ldots, y_n$, let $n_c = \sum_{i=1}^n \one \{ y_i = z_c \}$,
and let $\ROC_c$ denote the (classical) ROC curve associated with the
binary problem in \eqref{eq:real2binary}, for $c = 1, \ldots, m - 1$.

By Theorem 5 of \citet{Gneiting2018}, there is a natural bijection
between the class of the ROC curves and the class of the cumulative
distribution functions (CDFs) of Borel probability measures on the
unit interval.  In particular, any ROC curve can be associated with a
non-decreasing, right-continuous function $R : [0,1] \to [0,1]$ such
that $R(0) = 0$ and $R(1) = 1$.  Hence, any convex combination of the
ROC curves $\ROC_1, \ldots, \ROC_{m-1}$ can also be associated with a
non-decreasing, right-continuous function on the unit interval.  It is
in this sense that we define the following; in a nutshell, the UROC
curve averages the traditional ROC curves of which the ROC movie is
composed.

\begin{definition}  \label{defn:UROC} 
For data of the form \eqref{eq:real}, the \textit{universal receiver
  operating characteristic}\/ (UROC) \textit{curve} is the curve
associated with the function
\begin{equation}  \label{eq:UROC} 
\sum_{c = 1}^{m-1} w_c \, \ROC_c
\end{equation} 
on the unit interval, with weights  
\begin{equation}  \label{eq:w} 
w_c = \left. \left( \, \sum_{i=1}^c n_i \sum_{i=c+1}^m n_i \right) \right/ 
\left( \sum_{i=1}^{m-1} \sum_{j=i+1}^m ( \hsp j-i) \hsp n_i n_j \right)
\end{equation} 
for $c = 1, \ldots, m - 1$.
\end{definition}

Importantly, the weights in \eqref{eq:w} depend on the data in
\eqref{eq:real} via the outcomes $y_1, \ldots, y_n$ only.  Thus, they
are independent of the feature values and can be used meaningfully in
order to compare and rank features.  Their specific choice is
justified in Theorems \ref{thm:CPA2L} and \ref{thm:CPA2cov} below.
Clearly, the weights are nonnegative and sum to one.  If $m = n$ then
$n_1 = \cdots = n_m = 1$, and \eqref{eq:w} reduces to
\begin{equation}  \label{eq:wnoties} 
w_c = 6 \, \frac{c(n-c)}{n(n^2-1)} \quad {\rm for} \quad c = 1, \ldots, n - 1; 
\end{equation}
so the weights are quadratic in the rank $c$ and symmetric about the
inner most rank(s), at which they attain a maximum.  As we will see,
our choice of weights has the effect that in this setting the area
under the UROC curve, to which we refer as a general coefficient of
predictive ability ($\CPA$), relates linearly to Spearman's rank
correlation coefficient, in the same way that $\AUC$ relates linearly
to Somers' $D$.

In Fig.~\ref{fig:ROCM} the UROC curves appear in the final static
screen, subsequent to the ROC movies.  Within each ROC movie, the
individual frames show the ROC curve $\ROC_c$ for the feature
considered.  Furthermore, we display the threshold $z_c$, the
\textit{relative}\/ weight from \eqref{eq:w} (the actual weight
normalized to the unit interval, i.e., we show $w_c/\max_{l = 1,
	\ldots, m-1} w_l$), and $\AUC_c$, respectively, for $c = 1, \ldots,
m - 1$.  Once more we emphasize that the use of ROC movies, UROC
curves, and $\CPA$ frees researchers from the need to select ---
typically, arbitrary --- threshold values and binarize, as mandated by
classical ROC analysis.

Of course, if specific threshold values are of particular substantive
interest, the respective ROC curves can be extracted from the ROC
movie, and it can be useful to plot $\AUC_c$ versus the associated
threshold value $z_c$.  Displays of this type have been introduced and
studied by \citet{Rosset2005}.

\section{Coefficient of predictive ability ($\CPA$)}  \label{sec:CPA}

We proceed to define the coefficient of predictive ability ($\CPA$) as
a general measure of potential predictive ability, based on notation
introduced in Sections \ref{sec:ROCM} and \ref{sec:UROC}.

\begin{definition}  \label{defn:CPA}
For data of the form \eqref{eq:real} and weights $w_1, \ldots,
w_{m-1}$ as in \eqref{eq:w}, the \textit{coefficient of predictive
  ability} ($\CPA$) is defined as
\begin{equation}  \label{eq:CPA} 
\CPA = \sum_{c=1}^{m-1} w_c \, \AUC_c.   
\end{equation} 
In words, $\CPA$ equals the area under the UROC curve.
\end{definition}

Importantly, ROC movies, UROC curves, and $\CPA$ satisfy a fundamental
requirement on any generalization of ROC curves and $\AUC$, in that
they reduce to the classical notions when applied to a binary problem,
whence $m = 2$ in \eqref{eq:UROC} and \eqref{eq:CPA}, respectively.
Another requirement that we consider essential is that, when both the
feature values $x_1, \ldots, x_n$ and the outcomes $y_1, \ldots, y_n$
are pairwise distinct, the value of a performance measure remains
unchanged if we transpose the roles of the feature and the outcome.
As we will see, this is true under our specific choice \eqref{eq:w} of
the weights $w_c$ in the defining formula \eqref{eq:CPA} for $\CPA$,
but is not true under other choices, such as in the case of equal
weights.

\subsection{Interpretation as a weighted probability}  \label{sec:interpretation} 

We now express $\CPA$ in terms of pairwise comparisons via the
function $s$ in \eqref{eq:s}.  To this end, we usefully change
notation for the data in \eqref{eq:real} and refer to the feature
values in class $c \in \{ 1, \ldots, m \}$ as $x_{ck}$, for $k = 1,
\ldots, n_c$.  Thus, we rewrite \eqref{eq:real} as
\begin{equation}  \label{eq:real2class} 
(x_{11}, z_1), \ldots, (x_{1n_1},z_1), \, \ldots, \, (x_{m1}, z_m),
\ldots, (x_{mn_m},z_m) \, \in \, \real \times \real,
\end{equation} 
where $z_1 < \cdots < z_m$ are the unique values of $y_1, \ldots, y_n$
and $n_c = \sum_{i=1}^n \one \{ y_i = z_c \}$, for $c = 1, \ldots, m$.

\begin{theorem}  \label{thm:CPA2L} 
For data of the form \eqref{eq:real2class}, 
\begin{equation}  \label{eq:CPA2L} 
\CPA = 
\frac{\sum_{i=1}^{m-1} \sum_{j=i+1}^m \sum_{k=1}^{n_i} \sum_{l=1}^{n_j} ( \hsp j-i) \, s(x_{ik},x_{jl})}
{\sum_{i=1}^{m-1} \sum_{j=i+1}^m ( \hsp j-i) \, n_i n_j}.
\end{equation} 
\end{theorem} 

\begin{proof} 
By \eqref{eq:MWU}, the individual $\AUC$ values satisfy 
\[
\AUC_c = 
\frac{1}{\sum_{i=1}^c n_i \sum_{i=c+1}^m n_i} 
\sum_{i=1}^c \sum_{j=c+1}^m \sum_{k=1}^{n_i} \sum_{l=1}^{n_j} s(x_{ik},x_{jl})
\]
for $c = 1, \ldots, m - 1$.  In view of \eqref{eq:w} and \eqref{eq:CPA}, summation yields 
\begin{align*}
\CPA 
& = \sum_{c=1}^{m-1} w_c \, \AUC_c \\
& = \frac{\sum_{c=1}^{m-1} \sum_{i=1}^c \sum_{j=c+1}^m \sum_{k=1}^{n_i} \sum_{l=1}^{n_j} s(x_{ik},x_{jl})}{\sum_{i=1}^{m-1} \sum_{j=i+1}^m ( \hsp j-i) \, n_i n_j} \\
& = \frac{\sum_{i=1}^{m-1} \sum_{j=i+1}^m \sum_{k=1}^{n_{i}} \sum_{l=1}^{n_j} ( \hsp j-i) \, s(x_{ik},x_{jl})}{\sum_{i=1}^{m-1} \sum_{j=i+1}^m ( \hsp j-i) \, n_i n_j},
\end{align*} 
as claimed. 
\end{proof} 

Thus, $\CPA$ is based on pairwise comparisons of feature values,
counting the number of concordant pairs in \eqref{eq:real2class},
adjusting to a count of $\frac{1}{2}$ if feature values are tied, and
weighting a pair's contribution by a class based distance, $j - i$,
between the respective outcomes, $z_j > z_i$.  In other words, $\CPA$
equals a weighted probability of concordance, with weights that grow
linearly in the class based distance between outcomes.

The specific form of $\CPA$ in \eqref{eq:CPA2L} invites comparison to
a widely used measure of discrimination in biomedical applications,
namely, the C \textit{index} \citep{Harrell1996, Pencina2004}
\begin{equation}  \label{eq:C} 
{\rm C} = 
\frac{\sum_{i=1}^{m-1} \sum_{j=i+1}^m \sum_{k=1}^{n_{i}} \sum_{l=1}^{n_j} s(x_{ik},x_{jl})}{\sum_{i=1}^{m-1} \sum_{j=i+1}^m n_i n_j}.   
\end{equation} 
If the outcomes are binary, both the C index and $\CPA$ reduce to
$\AUC$.  While $\CPA$ can be interpreted as a weighted probability of
concordance, C admits an interpretation as an unweighted probability,
whence \citet{Mason2009} recommend its use for administrative
purposes.  However, the weighting in \eqref{eq:CPA2L} may be more
meaningful, as concordance between feature--outcome pairs with
outcomes that differ substantially in rank tends to be of greater
practical relevance than concordance between pairs with alike
outcomes.  While $\CPA$ admits the appealing, equivalent
interpretation \eqref{eq:CPA} in terms of binary $\AUC$ values and the
area under the UROC curve, relationships of this type are unavailable
for the C index.

Subject to conditions, the C index relates linearly to Kendall's rank
correlation coefficient \citep{Somers1962, Pencina2004, Mason2009}.
In Section \ref{sec:Spearman} we demonstrate the same type of
relationship for $\CPA$ and Spearman's rank correlation coefficient,
thereby resolving a problem raised by \citet[p.~95]{Heagerty2005}.
Just as the C index bridges and generalizes $\AUC$ and Kendall's
coefficient, $\CPA$ bridges and nests $\AUC$ and Spearman's
coefficient, with the added benefit of appealing interpretations in
terms of the area under the UROC curve and rank based covariances.

\subsection{Representation in terms of covariances}  \label{sec:CPA2cov}

The key result in this section represents $\CPA$ in terms of the
covariance between the class of the outcome and the mid rank of the
feature, relative to the covariance between the class of the outcome
and the mid rank of the outcome itself.

The mid rank method handles ties by assigning the arithmetic average
of the ranks involved \citep{Woodbury1940, Kruskal1958}.  For
instance, if the third to seventh positions in a list are tied, their
shared \textit{mid rank}~is $\frac{1}{5} (3 + 4 + 5 + 6 + 7) = 5$.
This approach treats equal values alike and guarantees that the sum of
the ranks in any tied group is unchanged from the case of no ties.  As
before, if $y_i = z_j$, where $z_1 < \cdots < z_m$ are the unique
values of $y_1, \ldots, y_n$ in \eqref{eq:real}, we say that the
\textit{class}\/ of $y_i$ is $j$.  In brief, we express this as
$\cl(y_i) = j$.  Similarly, we refer to the mid rank of $x_i$ within
$x_1, \ldots, x_n$ as $\mrk(x_i)$.

\begin{theorem}  \label{thm:CPA2cov} 
Let the random vector\/ $(X,Y)$ be drawn from the empirical
distribution of the data in \eqref{eq:real} or \eqref{eq:real2class}.
Then
\begin{equation}  \label{eq:CPA2cov} 
\CPA = \frac{1}{2} \left( \frac{\cov(\cl(Y),\mrk(X))}{\cov(\cl(Y),\mrk(Y))} + 1 \right).
\end{equation}
\end{theorem}

\begin{proof}
Suppose that the law of the random vector\/ $(X,Y)$ is the empirical
distribution of the data in \eqref{eq:real}.  Based on the equivalent
representation in \eqref{eq:real2class}, we find that
\[
\frac{\cov(\cl(Y),\mrk(X))}{\cov(\cl(Y),\mrk(Y))} 
= \frac{\sum_{i=1}^m \sum_{k=1}^{n_i} i \hsp \hsp \mrk(x_{ik}) - \frac{1}{2} \hsp (n+1) \sum_{i=1}^m i \hsp n_i}
{\sum_{i=1}^m i \hsp n_i \left( \sum_{j=0}^{i-1} n_j + \frac{1}{2} (n_i+1) \right) - \frac{1}{2} \hsp (n+1) \sum_{i=1}^m i \hsp n_i}, 
\]
where $n_0 = 0$.  Consequently, we can rewrite \eqref{eq:CPA2cov} as
\begin{equation}  \label{eq:num.denom} 
\CPA = \frac{\sum_{i=1}^m \sum_{k=1}^{n_i} i \hsp \hsp \mrk(x_{ik}) 
	+ \sum_{i=1}^m i \hsp n_i \left( \sum_{j=0}^{i-1} n_j + \frac{1}{2} n_i - n - \frac{1}{2} \right)}
{\sum_{i=1}^m i \hsp n_i \left( 2 \hsp \sum_{j=0}^{i-1} n_j + n_i - n \right)}. 
\end{equation}
	
We proceed to demonstrate that the numerator and denominator in
\eqref{eq:CPA2L} equal the numerator and denominator in
\eqref{eq:num.denom}, respectively.  To this end, we first compare
feature values within classes and note that
\[
\sum_{i=1}^m \sum_{k=1}^{n_i} \sum_{l=1}^{n_i} i \hsp s(x_{il},x_{ik}) 
= \sum_{i=1}^m i \sum_{k=1}^{n_i} \left( n_i - k + \frac{1}{2} \right)
= \frac{1}{2} \sum_{i=1}^m i \hsp n_i^2; 
\]
for if the feature values in class $i$ are all distinct, the largest
one exceeds $n_i - 1$ others, the second largest exceeds $n_i - 2$
others, and so on, and analogously in case of ties.  We now show the
equality of the numerators in \eqref{eq:CPA2L} and
\eqref{eq:num.denom}, in that
\begin{align*}
& \sum_{i=1}^{m-1} \sum_{j=i+1}^m \sum_{k=1}^{n_i} \sum_{l=1}^{n_j} ( \hsp j-i) \, s(x_{ik},x_{jl}) \\
& = \sum_{i=1}^{m-1} \sum_{j=i+1}^m \sum_{k=1}^{n_i} \sum_{l=1}^{n_j} j \hsp s(x_{ik},x_{jl}) 
  - \sum_{i=1}^{m-1} \sum_{j=i+1}^m \sum_{k=1}^{n_i} \sum_{l=1}^{n_j} i \hsp s(x_{ik},x_{jl}) \\ 
& \hspace{20mm} + \sum_{j=1}^{m-1} \sum_{i=j+1}^m \sum_{k=1}^{n_j} \sum_{l=1}^{n_i} j \hsp s(x_{ik},x_{jl})
  - \sum_{j=1}^{m-1} \sum_{i=j+1}^m \sum_{k=1}^{n_j} \sum_{l=1}^{n_i} j \hsp s(x_{ik},x_{jl}) \\
& = \sum_{i=1}^m \sum_{\substack{j=1 \\ j \neq i}}^m \sum_{k=1}^{n_i} \sum_{l=1}^{n_j} j \hsp s(x_{ik},x_{jl}) 
  - \sum_{i=1}^{m-1} \sum_{j=i+1}^m \sum_{k=1}^{n_i} \sum_{l=1}^{n_j} i \left( s(x_{jl},x_{ik}) + s(x_{ik},x_{jl}) \right) \\
\intertext{}
& = \sum_{j=1}^m \sum_{l=1}^{n_j} j \left( \hsp \mrk(x_{jl}) - \frac{1}{2} \right) - \sum_{i=1}^m \sum_{k=1}^{n_i} \sum_{l=1}^{n_i} i \hsp s(x_{il},x_{ik}) 
  - \sum_{i=1}^{m-1} \sum_{j=i+1}^m i \hsp n_i n_j \\
& = \sum_{i=1}^m \sum_{k=1}^{n_i} i \hsp \mrk(x_{ik)} - \frac{1}{2} \sum_{i=1}^m i \hsp n_i - \frac{1}{2} \sum_{i=1}^m i \hsp n_i^2 
  - n \sum_{i=1}^{m-1} i \hsp n_i + \sum_{i=1}^{m-1} i \hsp n_i \sum_{j=0}^i n_j \\
& = \sum_{i=1}^m \sum_{k=1}^{n_i} i \hsp \mrk(x_{ik)} - \frac{1}{2} \sum_{i=1}^m i \hsp n_i - \frac{1}{2} \sum_{i=1}^m i \hsp n_i^2 
  - n \sum_{i=1}^m i \hsp n_i + \sum_{i=1}^m i \hsp n_i \sum_{j=0}^i n_j \\
& = \sum_{i=1}^m \sum_{k=1}^{n_i} i \hsp \mrk(x_{ik)} 
  + \sum_{i=1}^m i \hsp n_i \left( \sum_{j=0}^{i-1} n_j + \frac{1}{2} n_i - n - \frac{1}{2} \right).
\end{align*}
As for the denominators, 
\begin{align*}
& \sum_{i=1}^{m-1} \sum_{j=i+1}^m ( \hsp j-i) \, n_i n_j =
  \sum_{i=1}^{m-1} \sum_{j=i+1}^{m} j \hsp n_i n_j - \sum_{i=1}^{m-1} \sum_{j=i+1}^m i \hsp n_i n_j \\
& \hspace{4mm} = \sum_{i=1}^m i \hsp n_i \sum_{k=0}^{i-1} n_k - n \sum_{i=1}^{m-1} i \hsp n_i + \sum_{i=1}^{m-1} i \hsp n_i \sum_{k=1}^i n_k \\
& \hspace{4mm} = 2 \sum_{i=1}^m i n_i \sum_{k=0}^{i-1} n_k - n \sum_{i=1}^{m-1} i \hsp n_i + \sum_{i=1}^{m-1} i \hsp n_i^2 
  + \sum_{i=1}^{m-1} i \hsp n_i \sum_{k=0}^{i-1} n_k - \sum_{i=1}^m i \hsp n_i \sum_{k=0}^{i-1} n_k \\
& \hspace{4mm} = 2 \sum_{i=1}^m i \hsp n_i \sum_{k=0}^{i-1} n_k - n \sum_{i=1}^{m-1} i \hsp n_i + \sum_{i=1}^{m-1} i \hsp n_i^2 - nmn_m + mn_m^2 \\
& \hspace{4mm} = 2 \sum_{i=1}^m i \hsp n_i \sum_{k=0}^{i-1} n_k - n \sum_{i=1}^m i \hsp n_i + \sum_{i=1}^m i \hsp n_i^2 \\
& \hspace{4mm} = \sum_{i=1}^m i \hsp n_i \left( 2 \hsp \sum_{j=0}^{i-1} n_j + n_i - n \right), 
\end{align*}
whence the proof is complete. 
\end{proof}

Interestingly, the representation \eqref{eq:CPA2cov} in terms of rank
and class based covariances appears to be new even in the special case
when the outcomes are binary, so that $\CPA$ reduces to $\AUC$.  The
representation also sheds new light on the asymmetry of $\CPA$, in
that, in general, the value of $\CPA$ changes if we transpose the
roles of the feature and the outcome.  In contrast to customarily used
measures of bivariate association and dependence, which are
necessarily symmetric \citep{Neslehova2007, Reshef2011, Weihs2018},
$\CPA$ is directed when the outcome is binary or ordinal.  Thus,
$\CPA$ avoids a technical issue with the use of rank-based correlation
coefficients in discrete settings, namely, that perfect classifiers do
not reach the optimal values of the respective performance measures
\citep[p.~565]{Neslehova2007}.  However, in the case of no ties at
all, to which we tend now, $\CPA$ becomes symmetric, as one would
expect, given that the feature and the outcome are on equal footing
then.

\subsection{Relationship to Spearman's rank correlation coefficient}  \label{sec:Spearman}

Spearman's rank correlation coefficient $\rhoS$ for data of the form
\eqref{eq:real} is generally understood as Pearson's correlation
coefficient applied to the respective ranks \citep{Spearman1904}.  In
case there are no ties in either $x_1, \ldots, x_n$ nor $y_1, \ldots,
y_n$, the concept is unambiguous, and Spearman's coefficient can be
computed as
\begin{equation}  \label{eq:def_rho_S} 
\rhoS = 1 - \frac{6}{n(n^2-1)} \sum_{i=1}^n \left( \rk(x_i) - \rk(y_i)
\right)^2,
\end{equation} 
where $\rk(x_i)$ denotes the rank of $x_i$ within $x_1, \ldots, x_n$,
and $\rk(y_i)$ the rank of $y_i$ within $y_1, \ldots, y_n$,

In this setting CPA relates linearly to Spearman's rank correlation
coefficient $\rhoS$, in the very same way that $\AUC$ relates to
Somers' $D$ in \eqref{eq:SomersD}.

\begin{theorem}  \label{thm:Spearman}
In the case of no ties, 
\begin{equation}  \label{eq:rho_S} 
{\rm CPA} = \frac{1}{2} \left( \hsp \rhoS + 1 \right).
\end{equation} 
\end{theorem}

Indeed, in case there are no ties, both mid ranks and classes reduce
to ranks proper, and then \eqref{eq:rho_S} is readily identified as a
special case of \eqref{eq:CPA2cov}.  For an alternative proof, in the
absence of ties the weights $w_{c}$ in \eqref{eq:w} are of the form
\eqref{eq:wnoties}.  The stated result then follows upon combining the
defining equation \eqref{eq:UROC}, the equality stated at the bottom
of the left column of page 4 in \citet{Rosset2005}, and equation (5)
in the same reference.

Note that $\CPA$ becomes symmetric in this case, as its value remains
unchanged if we transpose the roles of the feature and the outcome.
Furthermore, if the joint distribution of a bivariate random vector
$(X,Y)$ is continuous, and we think of the data in \eqref{eq:real} as
a sample from the respective population, then, by applying Definition
\ref{defn:CPA} and Theorem \ref{thm:Spearman} in the large sample
limit, and taking \eqref{eq:wnoties} into account, we (informally)
obtain a population version of $\CPA$, namely,
\begin{equation}  \label{eq:S2CPA}  
\CPA = 6 \int_0^1 \alpha (1 - \alpha) \, \AUC_\alpha \, \rm{d} \alpha
= \frac{1}{2} \left( \hsp \rhoS + 1 \right),
\end{equation}
where $\AUC_\alpha$ is the population version of $\AUC$ for $(X, \one
\{ Y \geq q_\alpha \})$, with $q_\alpha$ denoting the
$\alpha$-quantile of the marginal law of $Y$.  We defer a rigorous
derivation of \eqref{eq:S2CPA} to future work and stress that, as both
$X$ and $Y$ are continuous here, their roles can be interchanged.

Under the assumption of multivariate normality, the population version
of Spearman's $\rhoS$ relates to Pearson's correlation coefficient $r$
as
\begin{equation}  \label{eq:P2S}
\rhoS = \frac{6}{\pi} \arcsin \frac{r}{2};   
\end{equation}
see, e.g., \citet{Kruskal1958}.  Returning to the example in Section
\ref{sec:AUC}, where $(Y, X, X', X'')$ is jointly Gaussian with
covariance matrix \eqref{eq:MVN}, Table \ref{tab:MVN} states, for each
feature, the population values of Pearson's correlation coefficient
$r$, $\CPA$, and the C index relative to the real-valued outcome $Y$,
as derived from \eqref{eq:S2CPA} and \eqref{eq:P2S} and the respective
relationships for the C index and Kendall's rank correlation
coefficient $\tauK$, namely
\begin{equation}  \label{eq:K2C}  
{\rm C} = \frac{1}{2} \left( \hsp \tauK + 1 \right)
\end{equation}
and
\begin{equation}  \label{eq:P2K}  
\tauK = \frac{2}{\pi} \arcsin r. 
\end{equation}
These results imply that for a bivariate Gaussian population with
Pearson correlation coefficient $r \in (0,1)$ it is true that $\tauK >
\rhoS > 0$ and $\CPA > {\rm C} > 1/2$.  In fact, under positive
dependence it always holds that $\tauK \geq \rhoS \geq 0$, as
demonstrated by \citet{Caperaa1993}, whence $\CPA \geq {\rm C} \geq
1/2$.  However, there are also settings where these inequalities get
violated \citep{Schreyer2017}.  In Fig.~\ref{fig:ROCM} the $\CPA$
values for the features appear along with the UROC curves in the final
static screen, subsequent to the ROC movie.  The empirical values show
the expected approximate agreement with the population quantities in
the table.

\renewcommand{\arraystretch}{1.3}
\begin{table}[t]
\centering
\caption{Population values of Pearson's correlation coefficient $r$,
  $\CPA$, and the C index for the features $X$, $X'$, and $X''$
  relative to the real-valued outcome $Y$, where $(Y, X, X', X'')$ is
  Gaussian with covariance matrix \eqref{eq:MVN}.  \label{tab:MVN}}
\begin{tabular}{cccc}
\toprule 
Feature & $r$ & $\CPA$ & C \\
\toprule 
$X$   & 0.800 & 0.893 & 0.795 \\
$X'$  & 0.500 & 0.741 & 0.667 \\
$X''$ & 0.200 & 0.596 & 0.564 \\
\bottomrule 
\end{tabular} 
\end{table} 

Suppose now that the values $y_1, \ldots, y_n$ of the outcomes are
unique, whereas the feature values $x_1, \ldots, x_n$ might involves
ties.  Let $p \geq 0$ denote the number of tied groups within $x_1,
\ldots, x_n$.  If $p = 0$ let $V = 0$.  If $p \geq 1$, let $v_j$ be
the number of equal values in the $j$th group, for $j = 1, \ldots, p$,
and let
\[
V = \frac{1}{12} \sum_{j=1}^p \left( v_j^3 - v_j \right). 
\]
Then Spearman's \textit{mid rank adjusted}\/ coefficient
$\rhoM$ is defined as
\begin{equation}  \label{eq:def_rho_M}
\rhoM = 1 - \frac{6}{n(n^2-1)} \left( \sum_{i=1}^n \left( \, \mrk(x_i) - \rk(y_i) \right)^2 + V \right),
\end{equation}
where $\mrk$ is the aforementioned mid rank.  As shown by
\citet{Woodbury1940}, if one assigns all possible combinations of
integer ranks within tied sets, computes Spearman's $\rhoS$ in
\eqref{eq:def_rho_S} on every such combination and averages over the
respective values, one obtains the formula for $\rhoM$ in
\eqref{eq:def_rho_M}.

The following result reduces to the statement of Theorem
\ref{thm:Spearman} in the case $p = 0$ when there are no ties in $x_1,
\ldots, x_n$ either.

\begin{theorem}  \label{thm:Spearman.m}
In case there are no ties within\/ $y_1, \ldots, y_n$, 
\begin{equation}  \label{eq:rho_M}
{\rm CPA} = \frac{1}{2} \left( \hsp \rhoM + 1 \right).
\end{equation} 
\end{theorem}

\begin{proof}
As noted, $\rhoM$ arises from $\rhoS$ if one assigns all possible
combinations of integer ranks within tied sets, computes $\rhoS$ on
every such combination and averages over the respective values.  In
view of \eqref{eq:num.denom}, if there are no ties in $y_1, \ldots,
y_n$, averaging $\frac{1}{2} \left( \hsp \rhoS + 1 \right)$ over the
combinations yields $\frac{1}{2} \left( \hsp \rhoM + 1 \right)$, which
equals CPA by \eqref{eq:CPA2cov}.
\end{proof} 

The relationships \eqref{eq:SomersD}, \eqref{eq:rho_S} and
\eqref{eq:rho_M} constitute but special cases of the general,
covariance based representation \eqref{eq:CPA2cov}.  In this light,
CPA provides a unified way of quantifying potential predictive ability
for the full gamut of dichotomous, categorical, mixed
discrete-continuous and continuous types of outcomes.  In particular,
CPA bridges and generalizes $\AUC$, Somers' $D$ and Spearman's rank
correlation coefficient, up to a common linear relationship.

\subsection{Comparison of CPA to the C index and related measures}  \label{sec:comparison} 

We proceed to a more detailed comparison of the CPA measure
\eqref{eq:CPA} to the C index \eqref{eq:C} and measures studied by
\citet{Waegeman2008}.\footnote{We denote the measures $\widehat{U}$,
  $\widehat{U}_{\scriptscriptstyle \rm pairs}$,
  $\widehat{U}_{\scriptscriptstyle \rm ovo}$, and
  $\widehat{U}_{\scriptscriptstyle \rm cons}$ in equations (8), (16),
  (17), and (18) of \citet{Waegeman2008} by $U$, $\Upairs$, $\Uovo$,
  and $\Ucons$, respectively.}  As noted, both CPA and the C index are
rank-based, reduce to $\AUC$ when the outcome is binary, and become
symmetric when both the features and the outcomes are pairwise
distinct.  We relax these conditions slightly and restrict attention
to measures that use ranks only, reduce to $\AUC$ when the outcome is
binary \textit{and}\/ there are no ties in the feature values, and
become symmetric when there are no ties at all.  This excludes
measures based on the receiver error characteristic
\citep[REC,][]{Bi2003} and the regression receiver operating
characteristic \citep[RROC,][]{Orallo2013} curve, which are neither
rank based nor reduce to $\AUC$.  The $\Ucons$ measure of
\citet{Waegeman2008} averages consecutive $\AUC$ values in the same
fashion as $\CPA$ in \eqref{eq:CPA}, but uses constant weights, as
opposed to the class dependent weights \eqref{eq:w} for $\CPA$, and
does not become symmetric when there are no ties at all.\footnote{To
  see that $\Ucons$ does not become symmetric when there are no ties
  in $x_1, \ldots, x_n$ nor $y_1, \ldots, y_n$, consider a dataset of
  size $n \geq 4$, where $y_1 < \cdots < y_n$ and $x_3 < x_1 < x_2 <
  x_4 < \cdots < x_n$.  Then $\AUC_1 = (n-3)/(n-1)$, $\AUC_2 =
  (2n-5)/(2n-4)$, and $\AUC_c = 1$ for $c = 3, \ldots, n - 1$, whereas
  if we interchange the roles of the feature and the outcome, then
  $\AUC_1 = (n-2)/(n-1)$, $\AUC_2 = (2n-6)/(2n-4)$, and $\AUC_c = 1$
  for $c = 3, \ldots, n - 1$, resulting in distinct unweighted sums.}
The $\Upairs$ and $\Uovo$ measures of \citet{Waegeman2008} satisfy our
criteria, relate closely to the C index, and in the simulation setting
of Fig.~\ref{fig:CPA_Cindex} it holds that $\Uovo = \Upairs = {\rm
  C}$.\footnote{The $\Upairs$ measure corresponds to a performance
  criterion proposed by \citet[][equation (7.11)]{Herbrich2000} and
  equals the proportion of correctly ranked pairs of instances.
  Except for the treatment of ties in the feature, $\Upairs$ equals
  the C index.  In particular, if the feature values are pairwise
  distinct then $\Upairs = {\rm C}$.  The measure $\Uovo$ represents
  the \citet{Hand2001} approach of averaging the $\binom{m}{2}$
  one-versus-one $\AUC$ values in an $m$-class problem.  It has been
  compared to $\Upairs$ by \citet{Waegeman2008} and relates to the C
  index as well.  In particular, if the feature values are pairwise
  distinct and the dataset furthermore is balanced with class
  memberships $n_1 = \cdots = n_m$, as in the simulation setting that
  we report on in Fig.~\ref{fig:CPA_Cindex}, then $\Uovo = \Upairs =
  {\rm C}$.}

\begin{figure}[t]
\centering
\includegraphics[width = 0.85 \textwidth]{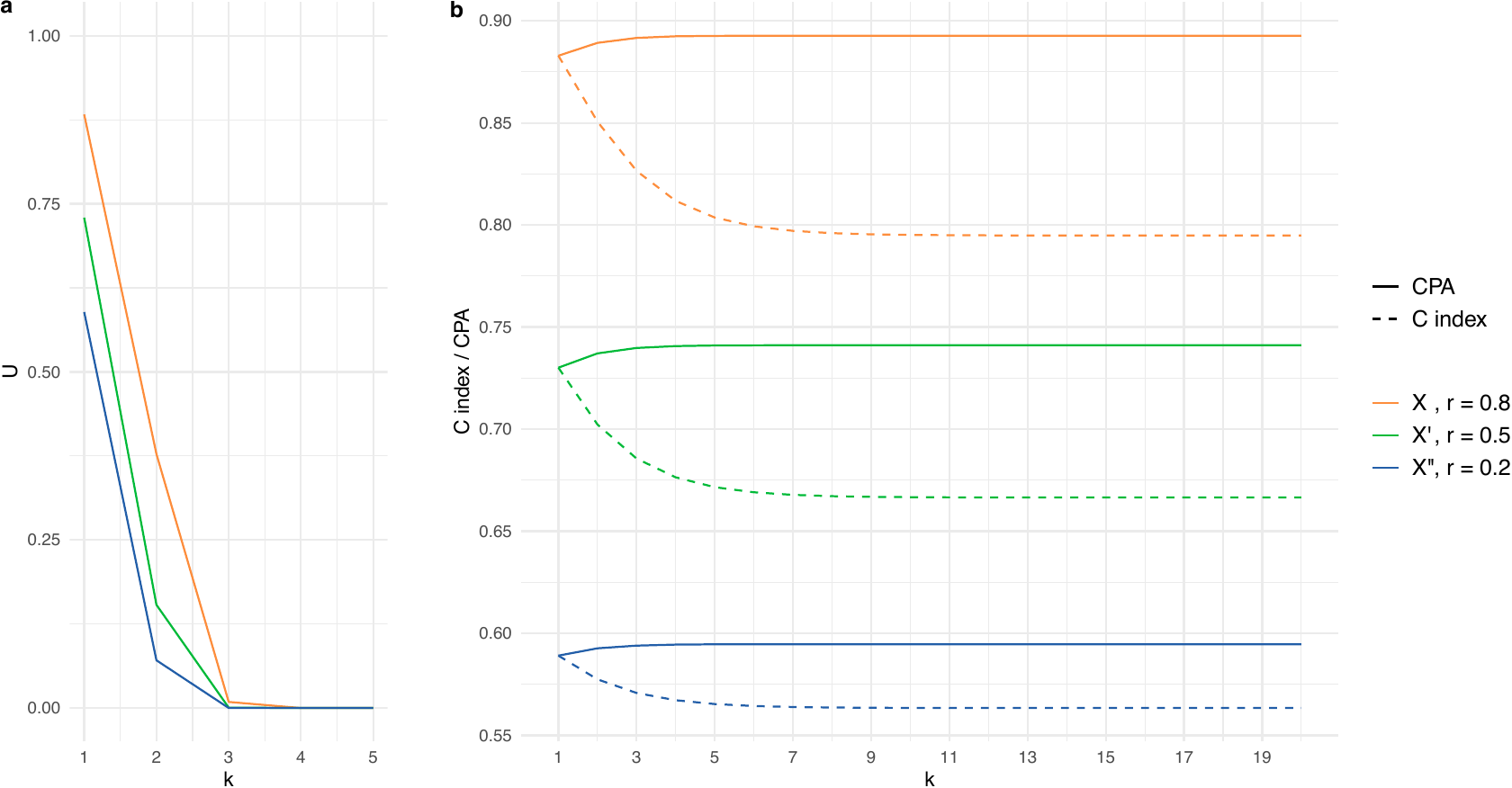}
\caption{Rank based performance measures for the features $X$, $X'$
  and $X''$ as predictors of the real-valued outcome $Y$ in the
  simulation example of Section \ref{sec:AUC}, with Pearson
  correlation coefficient $r = 0.8$, $0.5$ and $0.2$, respectively,
  based on a sample of size $n = 2^{20}$.  We discretize the
  continuous outcome into $2^k$ consecutive blocks of size $2^{20-k}$
  each, and plot (a) $U$, and (b) $\CPA$ and the C index as functions
  of the discretization level $k = 1, \ldots, 20$.  Note that $k = 1$
  yields a binary outcome and $k = 20$ a continuous
  outcome.  \label{fig:CPA_Cindex}}
\end{figure}

In view of the above requirements and properties, we restrict the
subsequent comparison to $\CPA$, the C index, and the $U$ measure
introduced by \citet{Waegeman2008}.  For a dataset with $m$ classes
$U$ equals the proportion of sequences of $m$ instances, one of each
class, that align correctly with the feature values.  As noted, these
measures are rank based and reduce to $\AUC$ when the outcome is
binary and there are no ties in the feature values.  In the continuous
case with no ties in the feature values nor in the outcomes, they
become symmetric, $U$ attains the value 1 under a perfect ranking and
the value 0 otherwise, ${\rm C} = \frac{1}{2} \left( 1 - \tauK
\right)$, and $\CPA = \frac{1}{2} \left( 1 - \rhoS \right)$.

In Fig.~\ref{fig:CPA_Cindex} we report on a simulation experiment
where we draw samples of $2^{20}$ instances from the joint Gaussian
distribution of the random vector $(Y, X, X', X'')$ with covariance
matrix \eqref{eq:MVN}, so that the features have Pearson correlation
coefficient $r = 0.8$, $0.5$, and $0.2$ with the continuous outcome
$Y$.  By discretizing the outcome into $2^k$ consecutive blocks of
size $2^{20-k}$ each, where $k = 1, \ldots, 20$, and computing $\CPA$,
the C index and the $U$ measure as a function of $k$, all
discretization levels are considered, ranging from a binary variable
for $k = 1$ to continuous outcomes for $k = 20$.  When $k = 1$ the
three measures coincide and equal $\AUC$, essentially at the
population value of
\begin{equation}  \label{eq:AUC1/2} 
\AUC_{1/2} = \frac{2}{\pi} \arcsin \frac{r}{\sqrt{2}} + \frac{1}{2},  
\end{equation} 
in the sense stated subsequent to \eqref{eq:S2CPA}.  The $U$ measure
is tailored to ordinal outcomes with a few classes only and
degenerates rapidly with $k$.  When $k = 20$, $\CPA$ and the C index
are rescaled versions of Spearman's $\rhoS$ and Kendall's $\tauK$,
essentially at the population values in Table \ref{tab:MVN}.

Throughout, the measures lie in between their common value for $k =
1$, which equals $\AUC$, and the respective values for $k = 20$.  For
all features and all $k > 1$, the C index is smaller than $\CPA$, and
$\CPA$ varies considerably less with the discretization level than the
C index.  To supplement these experiments with an analytic
demonstration, suppose that $X$ and $Y$ are bivariate Gaussian with
nonnegative Pearson correlation $r$.  If we convert $Y$ to a balanced
binary outcome, then both $\CPA$ and the C index reduce to a common
value, namely, $\AUC_{1/2}$ in \eqref{eq:AUC1/2}.  As a function of
$r$, the ratio of the C index for the continuous vs.~the balanced
binary outcome attains values between 0.8996 and 1, whereas for $\CPA$
the respective ratio remains between 1 and 1.0156, as illustrated in
Fig.~\ref{fig:ratio}.  These findings along with results in
\citet{Caperaa1993} and \citet{Schreyer2017} suggest that, quite
generally, $\CPA$ and the C index yield qualitatively similar results
in practice, with $\CPA$ being less sensitive to quantization effects,
and the value of $\CPA$ typically being larger than for the C index.

\begin{figure}[t]
\centering
\includegraphics[width = 0.80 \textwidth]{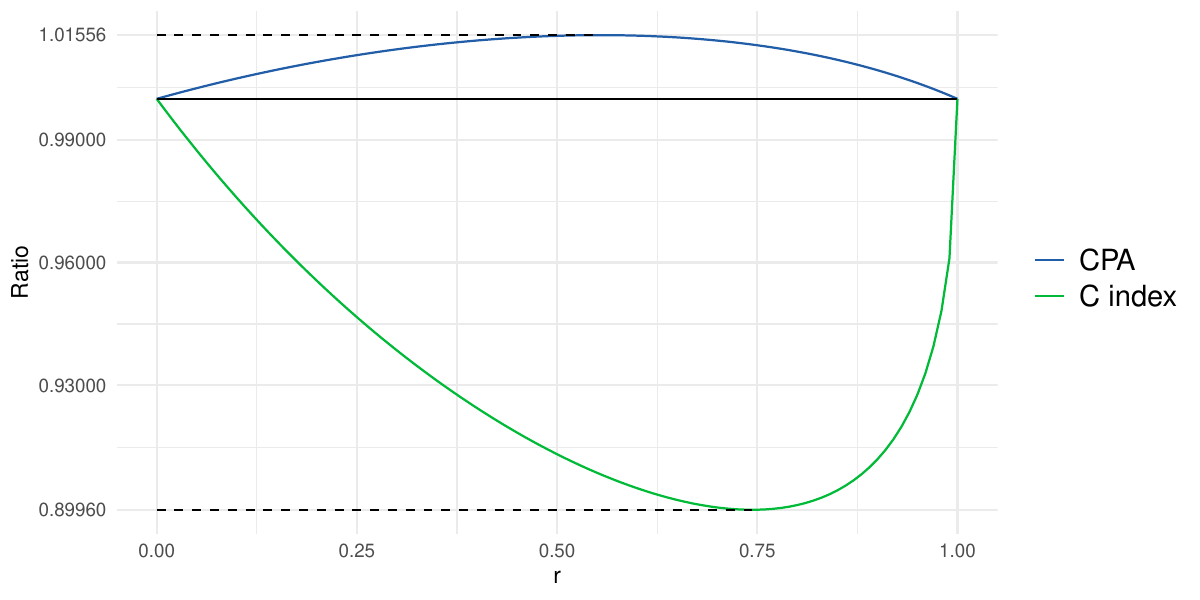}
\caption{Ratio of $\CPA$ (blue curve) respectively the C index (green
  curve) for the feature $X$ as a predictor of the continuous outcome
  $Y$ over $\AUC$ for $X$ and the balanced binary outcome $\one \{ Y
  \geq 0 \}$, where $X$ and $Y$ are bivariate Gaussian with Pearson
  correlation $r \in [0,1]$.  The solid horizontal line is at a ratio
  of 1, which is attained when $r = 0$ and $r = 1$. \label{fig:ratio}}
\end{figure}

\subsection{Computational issues}  

We turn to a discussion of the computational costs of generalized ROC
analysis for a dataset of the form \eqref{eq:real} or
\eqref{eq:real2class} with $n$ instances and $m \leq n$ classes.

It is well known that a traditional ROC curve can be generated from a
dataset with $n$ instances in $O(n \log n)$ operations
\citep[Algorithm 1]{Fawcett2006}.  A ROC movie comprises $m - 1$
traditional ROC curves, so in a na\"{i}ve approach, ROC movies can be
computed in $O(mn \log n)$ operations.  However, our implementation
takes advantage of recursive relations between consecutive component
curves ROC$_{i-1}$ and ROC$_i$.  While a formal analysis will need to
be left to future work, we believe that our algorithm has
computational costs of $O(n \log n)$ operations only.  If the number
$m$ of unique values of the outcome is large, then for all practical
purposes the ROC movie can be shown at a modest number $m_0$ of
distinct values only, at a computational cost of $O(m_0 n \log n)$
operations.  For example, in the setting of
Fig.~\ref{fig:ROCM_HRES_PST} in the meteorological case study in
Section \ref{sec:NWP} there are $m = 35,993$ unique values of the
outcome, whereas the ROC movie uses $m_0 = 401$ frames only.  For the
vertical averaging of the component curves in the construction of UROC
curves, we partition the unit interval into 1,000 equally sized
subintervals.

Importantly, $\CPA$ can be computed in $O(n \log n)$ operations,
without any need to invoke ROC analysis, by sorting $x_1, \ldots, x_n$
and $y_1, \ldots, y_n$, computing the respective mid ranks and
classes, and plugging into the rank based representation
\eqref{eq:num.denom}.  Similarly, there are algorithms for the
computation of the C index in $O(n \log n)$ operations
\citep{Knight1966, Christensen2005}.

\subsection{Key properties: Comparison to traditional ROC analysis}  \label{sec:summary}

We are now in a position to judge whether the proposed toolbox of ROC
movies, UROC curves, and $\CPA$ constitutes a proper generalization of
traditional ROC analysis.  To facilitate the assessment, the
subsequent statements admit immediate comparison with the key insights
of classical ROC analysis, as summarized in Section
\ref{sec:initialsummary}.

We start with the trivial but important observation that the new tools
nest the notions of traditional ROC analysis.  This is not to be taken
for granted, as extant generalizations do not necessarily share this
property.
\begin{itemize} 
\item[(0)] In the case of a binary outcome, both the ROC movie and the
  UROC curve reduce to the ROC curve, and $\CPA$ reduces to $\AUC$.
\item[(1)] ROC movies, the UROC curve and $\CPA$ are straightforward
  to compute and interpret, in the (rough) sense of \textit{the larger
    the better}.
\item[(2)] $\CPA$ attains values between 0 and 1 and relates linearly
  to the covariance between the class of the outcome and the mid rank
  of the feature, relative to the covariance between the class and the
  mid rank of the outcome.  In particular, if the outcomes are
  pairwise distinct, then $\CPA = \frac{1}{2} \left( \hsp \rho_{\rm M}
  + 1 \right)$, where $\rho_{\rm M}$ is Spearman's mid rank adjusted
  coefficient \eqref{eq:def_rho_M}.  If the outcomes are binary, then
  $\CPA = \frac{1}{2} \left( D + 1 \right)$ in terms of Somers' $D$.
  For a perfect feature, $\CPA = 1$, $\rho_{\rm M} = 1$ under pairwise
  distinct and $D = 1$ under binary outcomes.  For a feature that is
  independent of the outcome, $\CPA = \frac{1}{2}$, $\rho_{\rm M} = 0$
  under pairwise distinct and $D = 0$ under binary outcomes.
\item[(3)] The numerical value of $\CPA$ admits an interpretation as a
  weighted probability of concordance for feature--outcome pairs, with
  weights that grow linearly in the class based distance between
  outcomes.
\item[(4)] ROC movies, UROC curves, and $\CPA$ are purely rank based
  and, therefore, invariant under strictly increasing transformations.
  Specifically, if $\phi : \real \to \real$ and $\psi : \real \to
  \real$ are strictly increasing, then the ROC movie, UROC curve, and
  $\CPA$ computed from
  \begin{equation}  \label{eq:realT} 
  (\phi(x_1),\psi(y_1)), \ldots, (\phi(x_n),\psi(y_n)) \, \in \, \real \times \real
  \end{equation} 
  are the same as the ROC movie, UROC curve, and $\CPA$ computed from
  the data in \eqref{eq:real}.
\end{itemize} 
We iterate and emphasize that, as an immediate consequence of the
final property, ROC movies, UROC curves, and $\CPA$ assess the
discrimination ability or \textit{potential}\/ predictive ability of a
point forecast, regression output, feature, marker, or test.  Markedly
different techniques are called for if one seeks to assess a
forecast's \textit{actual}\/ value in any given applied problem
\citep{Bouallegue2015, Ehm2016}.

\section{Real data examples}  \label{sec:examples}

In the following examples from survival analysis and numerical weather
prediction the usage of ROC movies, UROC curves, and $\CPA$ is
demonstrated.  We start by returning to the survival example from
Section \ref{sec:introduction}, where the new set of tools frees
researchers form the need to artificially binarize the outcome.  Then
the use of $\CPA$ is highlighted in a study of recent progress in
numerical weather prediction (NWP), and in a comparison of the
predictive performance of NWP models and convolutional neural
networks.

\subsection{Survival data from Mayo Clinic trial}  \label{sec:survival}

In the introduction, Figs.~\ref{fig:ROCC_survival} and
\ref{fig:ROCM_survival} serve to illustrate and contrast traditional
ROC curves, ROC movies and UROC curves.  They are based on a classical
dataset from a Mayo Clinic trial on primary biliary cirrhosis (PBC), a
chronic fatal disease of the liver, that was conducted between 1974
and 1984 \citep{Dickson1989}.  The data are provided by various
\textsf{R} packages, such as \texttt{SMPracticals} and
\texttt{survival}, and have been analyzed in textbooks
\citep{Fleming1991, Davison2003}.  The outcome of interest is survival
time past entry into the study.  Patients were randomly assigned to
either a placebo or treatment with the drug D-penicillamine.  However,
extant analyses do not show treatment effects \citep{Dickson1989}, and
so we follow previous practice and study treatment and placebo groups
jointly.

We consider two biochemical markers, namely, serum albumin and serum
bilirubin concentration in mg/dl, for which higher and lower levels,
respectively, are known to be indicative of earlier disease stages,
thus supporting survival.  Hence, for the purposes of ROC analysis we
reverse the orientation of the serum bilirubin values.  Given our goal
of illustration, we avoid complications and remove patient records
with censored survival times, to obtain a dataset with $n = 161$
patient records and $m = 156$ unique survival times.  The proper
handling of censoring is beyond the scope of our study, and we leave
this task to subsequent work.  For a discussion and comparison of
extant approaches to handling censored data in the context of
time-dependent ROC curves see \citet{Blanche2013}.

The traditional ROC curves in Fig.~\ref{fig:ROCC_survival} are
obtained by binarizing survival time at a threshold of 1462 days,
which is the survival time in the data record that gets closest to
four years.  The ROC movies and UROC curves in
Fig.~\ref{fig:ROCM_survival} are generated directly from the survival
times, without any need to artificially pick a threshold.  The $\CPA$
values for serum albumin and serum bilirubin are $0.73$ and $0.77$,
respectively, and contrary to the ranking in
Fig.~\ref{fig:ROCC_survival}, where bilirubin was deemed superior,
based on outcomes that were artificially made binary.  Our tools free
researchers from the need to binarize, and still they allow for an
assessment at the binary level, if desired.  For example, the ROC
curves and $\AUC$ values from Fig.~\ref{fig:ROCC_survival} appear in
the ROC movie at a threshold value of 1462 days.  In line with current
uses of $\AUC$ in a gamut of applied settings, $\CPA$ is particularly
well suited to the purposes of feature screening and variable
selection in statistical and machine learning models
\citep{Guyon2003}.  Here, $\AUC$ and $\CPA$ demonstrate that both
albumin and bilirubin contribute to prognostic models for survival
\citep{Dickson1989,Fleming1991}.

\begin{figure}[t] 
\centering
\includegraphics[width = \textwidth]{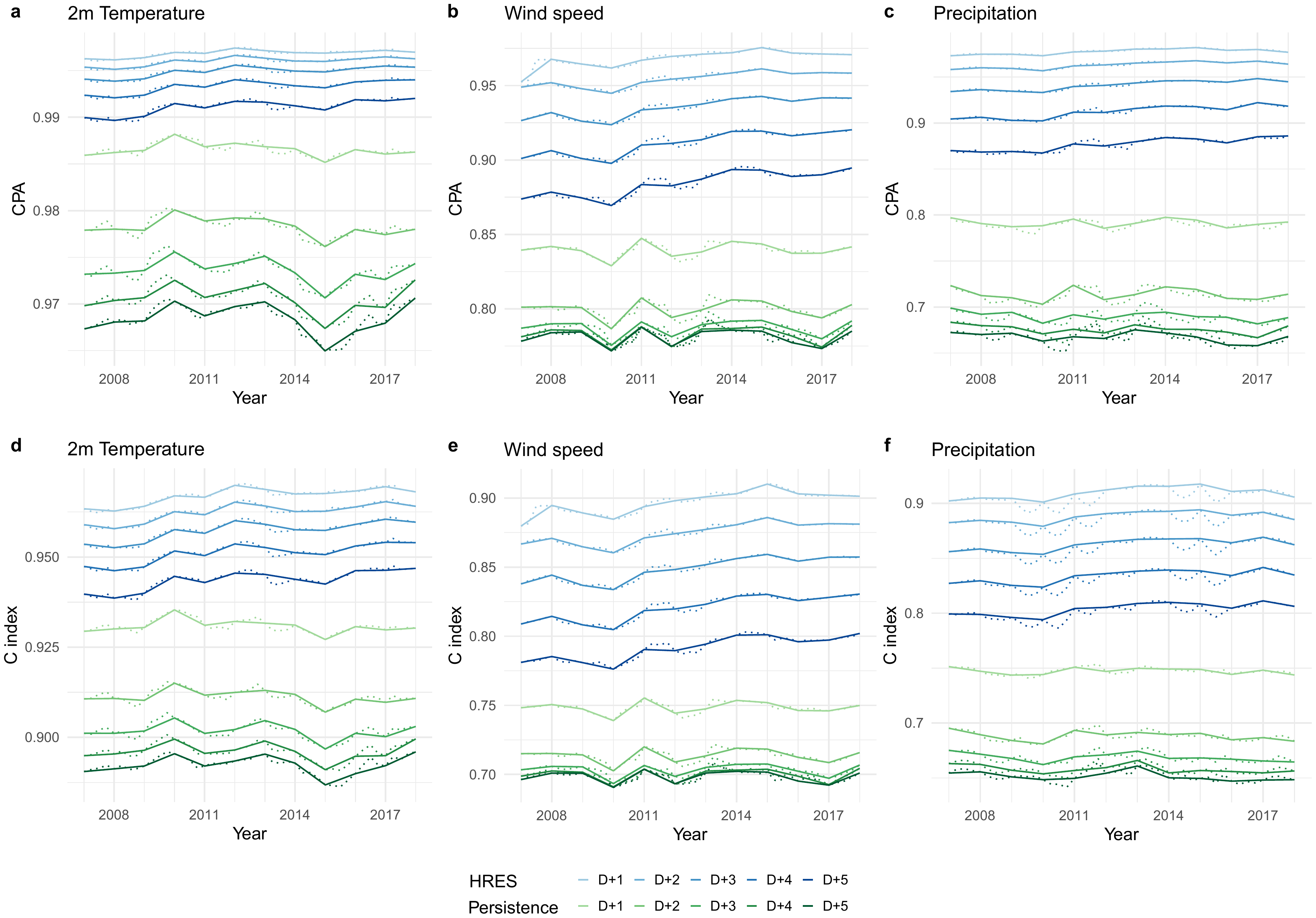}
\caption{Temporal evolution of $\CPA$ and the C index for forecasts
  from the ECMWF high-resolution model at lead times of one to five
  days in comparison to the simplistic persistence forecast in terms
  of $\CPA$ (a, b, c) and the C index (d, e, f).  The weather
  variables considered are (a, d) surface (2-meter) temperature, (b,
  e) surface wind speed and (c, f) 24-hour precipitation accumulation.
  The measures refer to a domain that covers Europe and twelve-month
  periods that correspond to January--December (solid and dotted
  lines), April--March, July--June and October--September (dotted
  lines only), based on gridded forecast and observational data from
  January 2007 through December 2018.  \label{fig:CPA_HRES_PST}}
\end{figure}

\begin{figure}[t]
\centering
\animategraphics[label = myAnim7, height = 75mm]{28}{figures/Test4_reduced3}{}{}
\caption{ROC movies, UROC curves, and $\CPA$ for ECMWF high-resolution
  (HRES) and persistence forecasts of 24-hour precipitation
  accumulation over Europe at a lead time of five days in calendar
  year 2018.  In the ROC movies, the number at upper left shows the
  threshold at hand in the unit of millimeter, the number at upper
  center the relative weight $w_c/\max_{l = 1, \ldots, m-1} w_l$ from
  \eqref{eq:w}, and the numbers at bottom right the respective $\AUC$
  values.  \label{fig:ROCM_HRES_PST}}
\end{figure}

\subsection{Monitoring progress in numerical weather prediction (NWP)}  \label{sec:NWP}

Here we illustrate the usage of $\CPA$ in the assessment of recent
progress in numerical weather prediction (NWP), which has experienced
tremendous advance over the past few decades \citep{Bauer2015,
	Alley2019, Bouallegue2019}.  Specifically, we consider forecasts of
surface (2-meter) temperature, surface (10-meter) wind speed and
24-hour precipitation accumulation initialized at 00:00 UTC at lead
times from a single day (24 hours) to five days (120 hours) ahead from
the high-resolution model operated by the European Centre for
Medium-Range Weather Forecasts \citep{ecmwf2012}, which is generally
considered the leading global NWP model.  The forecast data are
available at \url{https://confluence.ecmwf.int/display/TIGGE}.  As
observational reference we take the ERA5 reanalysis product
\citep{Hersbach2018}.  We use forecasts and observations from $279
\times 199 = 55,521$ model grid boxes of size $0.25^\circ \times
0.25^\circ$ each in a geographic region that covers Europe from
25.0$^\circ$ W to 44.5$^\circ$ E in latitude and 25.0$^\circ$ N to
74.5$^\circ$ N in longitude.  The time period considered ranges from
January 2007 to December 2018.

In Fig.~\ref{fig:CPA_HRES_PST} we apply $\CPA$ and the C index to
compare forecasts from the ECMWF high-resolution run to a reference
technique, namely, the persistence forecast.  The persistence forecast
is simply the most recent available observation for the weather
quantity of interest; as such, the forecast value does not depend on
the lead time.  $\CPA$ and the C index are computed on rolling
twelve-month periods that correspond to January--December,
April--March, July--June or October--September, typically comprising
$n = 365 \times 55,521 = 20,265,165$ individual forecast cases.  The
ECMWF forecast has considerably higher $\CPA$ and C index than the
persistence forecast for all lead times and variables considered.  For
the persistence forecast the measures fluctuate around a constant
level; for the ECMWF forecast they improve steadily, attesting to
continuing progress in NWP \citep{Bauer2015, Alley2019,
  Bouallegue2019, Haiden2021}.

To place these findings further into context, recall that $\CPA$ is a
weighted average of $\AUC$ values for binarized outcomes at individual
threshold values, as have been used for performance monitoring by
weather centers \citep{Bouallegue2019, Haiden2021}.  The $\CPA$
measure preserves the spirit and power of classical ROC analysis, and
frees researchers from the need to binarize real-valued outcomes.
Results in terms of the C index are qualitatively similar, with the
numerical value of $\CPA$ being higher than for the C index.

The ROC movies, UROC curves, and $\CPA$ values in
Fig.~\ref{fig:ROCM_HRES_PST} compare the ECMWF high-resolution
forecast to the persistence forecast for 24-hour precipitation
accumulation at a lead time of five days in calendar year 2018.  As
noted, this record comprises more than 20 million individual forecast
cases, and there are $m = 35,993$ unique values of the outcome.  We
certainly lack the patience to watch the full sequence of $m - 1$
screens in the ROC movie.  A pragmatic solution is to consider a
subset $\cC \subseteq \{ 1, \ldots, m - 1 \}$ of indices, so that
$\ROC_c$ is included in the ROC movie (if and) only if $c \in \cC$.
Specifically, we set positive integer parameters $a \leq m - 1$ and
$b$ such that the ROC movie comprises at least $a$ and at most $a + b$
curves.  Let the integer $s$ be defined such that $1 + (a - 1) s \leq
m - 1$ $< 1 + as$, and let $\cC_a = \{ 1, 1 + s, \ldots, 1 + (a-1)s
\}$, so that $|\cC_a| = a$.  Let $\cC_b = \{ c : n_c \geq n/b \}$;
evidently, $|\cC_b| \leq b$.  Finally, let $\cC = \cC_a \cup \cC_b$ so
that $a \leq |\cC| \leq a + b$.  We have made good experiences with
choices of $a = 400$ and $b = 100$, which in
Fig.~\ref{fig:ROCM_HRES_PST} yield a ROC movie with 401 screens.

\subsection{WeatherBench: Convolutional neural networks (CNNs) vs.~NWP models}  \label{sec:WeatherBench} 

As noted, operational weather forecasts are based on the output of
global NWP models that represent the physics of the atmosphere.
However, the grid resolution of NWP models remains limited due to
finite computing resources \citep{Bauer2015}.  Spurred by the ever
increasing popularity and successes of machine learning models,
alternative, data-driven approaches are in vigorous development, with
convolutional neural networks \citep[CNNs;][]{LeCun2015} being a
particularly attractive starting point, due to their ease of
adaptation to spatio-temporal data.  \citet{Rasp2020} introduce
WeatherBench, a ready-to use benchmark dataset for the comparison of
data-driven approaches, such as CNNs and a classical linear regression
(LR) based technique, to NWP models, such as the aforementioned HRES
model and simplified versions thereof, T63 and T42, which run at
successively coarser resolutions.  Furthermore, WeatherBench supplies
baseline methods, including both the persistence forecast and
climatological forecasts.

As evaluation measure for the various types of point forecasts,
WeatherBench uses the root mean squared error (RMSE).  In related
studies, the RMSE is accompanied by the anomaly correlation
coefficient (ACC), i.e., the normalized product moment between the
difference of the forecast at hand and the climatological forecast,
and the difference between the outcome and the climatological forecast
\citep{Weyn2020}.  However, as noted by \citet{Rasp2020}, results in
terms of RMSE and ACC tend to be very similar.  Here we argue that a
rank based measure, such as $\CPA$ or the C index, would be a more
suitable companion measure to RMSE than ACC.

\begin{figure}[t]
\centering
\includegraphics[width = \textwidth]{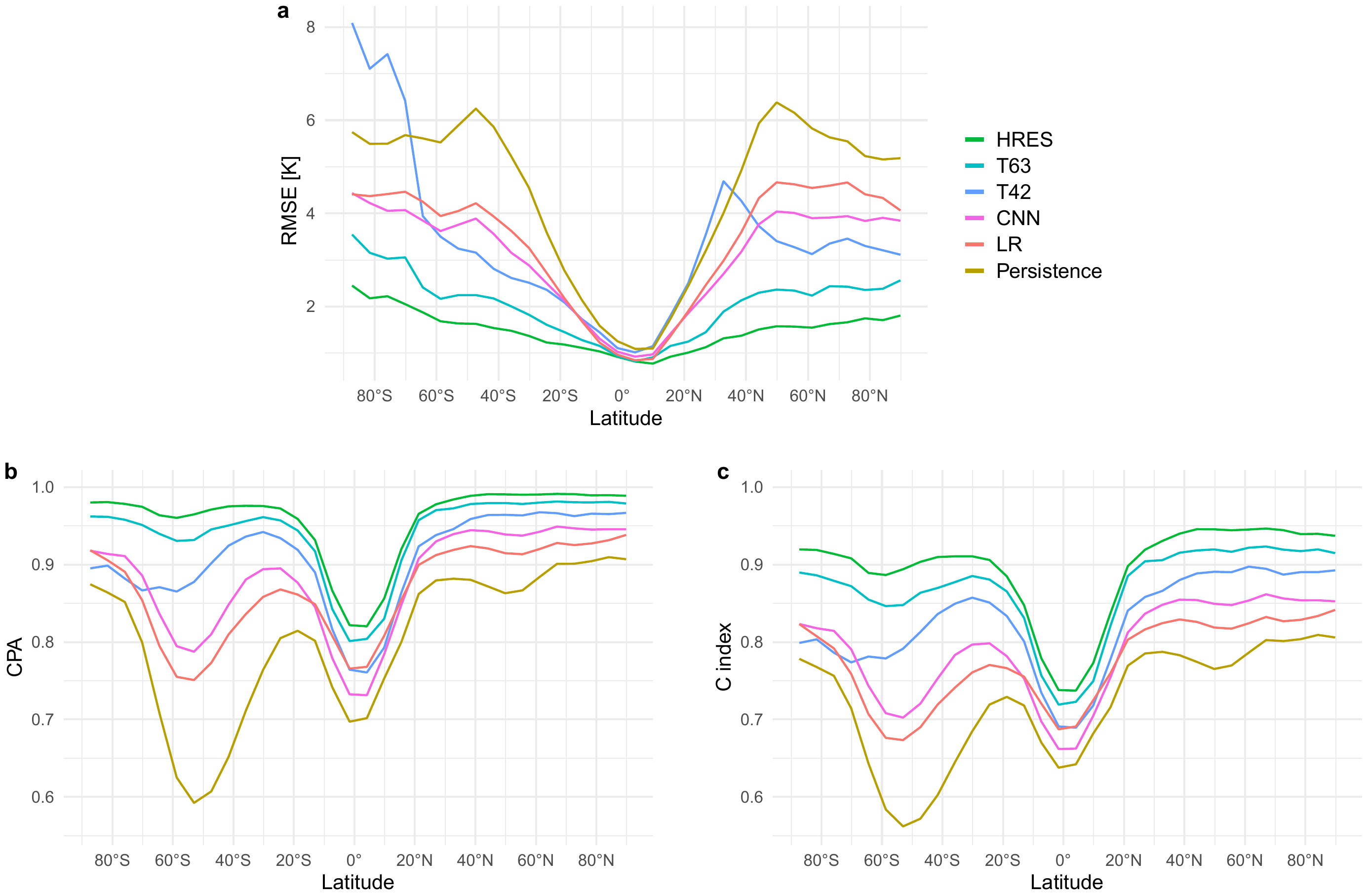}
\caption{Predictive ability of WeatherBench three days ahead forecasts
  of 850 hPa temperature in 2017 and 2018 at different latitudes in
  terms of (a) RMSE, (b) $\CPA$, and (c) the C index.  HRES, T63, and
  T42 indicate NWP models run at decreasing grid resolution that are
  compared to the CNN, linear regression (LR), and persistence
  forecasts \citep{Rasp2020}.  Note that RMSE is negatively oriented
  (the smaller the better), whereas $\CPA$ and the C index are
  positively oriented.  \label{fig:RMSE_CPA}}
\end{figure}        

Figure \ref{fig:RMSE_CPA} compares WeatherBench forecasts three days
ahead for temperature at 850 hPa pressure, which is at around 1.5 km
height, in terms of RMSE (in Kelvin), $\CPA$, and the C index.  With
reference to Table 2 of \citet{Rasp2020}, we consider the persistence
forecast, the (direct) linear regression (LR) forecast, the (direct)
CNN forecast, the Operational IFS (HRES) forecast, and successively
coarser versions thereof (T63 and T42).  The panels display the
performance measures as functions of latitude bands, from the South
Pole at 90$^\circ$S to the equator at 0$^\circ$ and the North Pole at
90$^\circ$N, for the WeatherBench final evaluation period of the years
2017 and 2018.  The measures are initially computed grid cell by grid
cell, and then averaged across the grid cells in a latitude band,
which is compatible with the latitude based weighting that is employed
in WeatherBench.  Note that RMSE is negatively oriented (the smaller,
the better), whereas the rank based measures are positively oriented
(the closer to the ideal value of 1 the better).

With respect to RMSE (Fig.~\ref{fig:RMSE_CPA}a) marked geographical
differences are visible.  In equatorial regions, where day-to-day
temperature variations are generally low, all forecasts have a low
RMSE and the range between the best-performing HRES forecast and the
simplistic persistence forecast is small.  The HRES forecast remains
best for all latitudes, followed by the T63 forecast.  The coarsest
dynamical model forecast, T42, shows a further deterioration as
expected, but with large outliers in the high latitudes of the
southern hemisphere and in the 30s of the northern hemisphere.  It is
likely that the lack of model orography creates large errors in areas
of high terrain such as the Antarctic plateau and the Himalayas.
Among the data-driven forecasts, CNN is better than LR for all
extratropical latitudes.  Finally, persistence performs worst through
all latitudes with prominent peaks near 50$^\circ$S and 50$^\circ$N.
These are the midlatitude storm track regions, where day-to-day
changes are large and impede good forecasts based on persistence.

\begin{figure}[t]
\centering 
\includegraphics[width = 0.90 \textwidth]{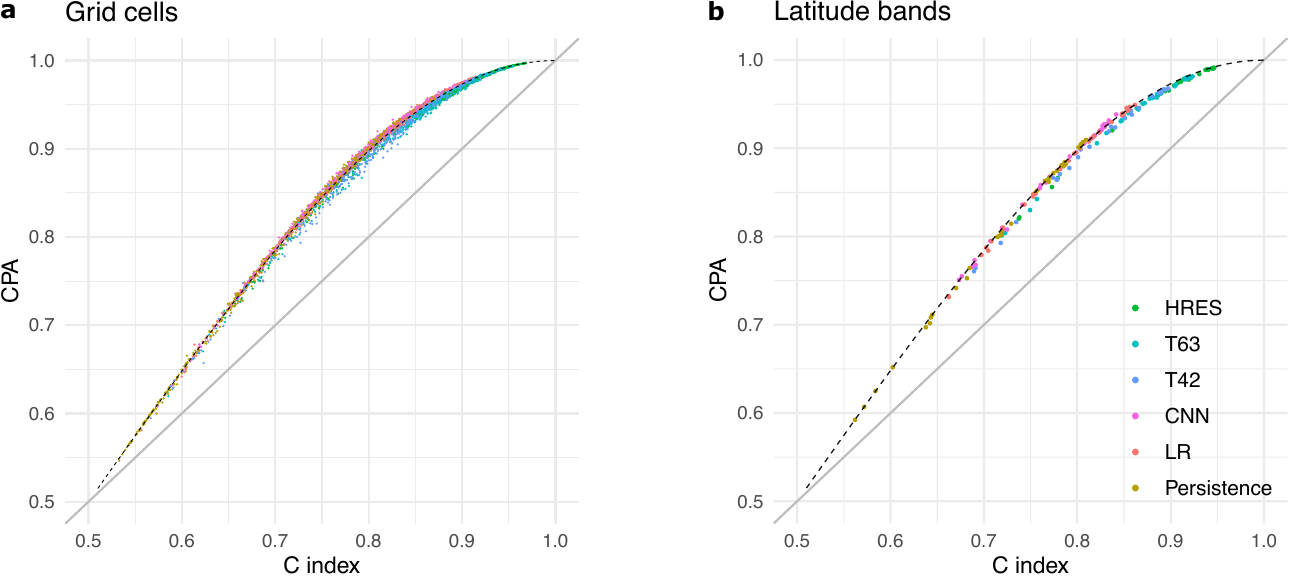}
\caption{Comparison of $\CPA$ and the C index for WeatherBench three
  days ahead forecasts of 850 hPa temperature in 2017 and 2018.  The
  points in the scatterplots of $\CPA$ vs.~the C index correspond to
  (a) measures for individual grid cells and (b) averages of measures
  over latitude bands.  The dashed curves show the theoretical
  relationship between $\CPA$ and the C index in bivariate Gaussian
  populations.  \label{fig:scatterplots}}
\end{figure}        

The corresponding results in terms of $\CPA$ and the C index
(Fig.~\ref{fig:RMSE_CPA}b--c) resemble each other, but show remarkable
differences to the RMSE based analysis.  Most notable are their low
values in the tropics, which indicate poor performance of all
forecasts, well in line with recent findings in meteorology
\citep{Kniffka2020}.  In contrast, the low RMSE suggests superior
performance in this region.  The rank based measures are independent
of magnitude and thus provide a scale free assessment of
predictability.  Another striking difference to RMSE is the large drop
in the Furious Fifties of the southern hemisphere, creating a large
asymmetry with the northern midlatitudes.  This area is almost
entirely oceanic and characterized by mobile low-pressure systems, the
dynamical behaviour of which appears to be difficult to learn under
data-driven approaches.

In Fig.~\ref{fig:scatterplots} we compare $\CPA$ and the C index, both
for individual grid cells and for measures that have been averaged
over latitude bands.  The scatterplots illustrate the findings from
Sections \ref{sec:Spearman} and \ref{sec:comparison}, in that the
value of $\CPA$ throughout is larger than for the C index, in
remarkably close agreement with the respective theoretical
relationship under the assumption of bivariate Gaussianity.

We conclude that RMSE and the rank based measures bring orthogonal
facets of predictive performance to researchers' attention, and
encourage the usage of of $\CPA$ or the C index to supplement RMSE as
key performance measures in WeatherBench.  While ACC is scale free as
well, it is moment based rather than rank based, and thus is more
closely aligned with RMSE than a rank based measure.  Similar
recommendations apply in many practical settings, where predictions of
a real-valued outcome are evaluated, and a magnitude dependent
measure, such as RMSE, is usefully accompanied by a rank based
criterion of predictive performance.  In the special case of
probabilistic classifiers for binary outcomes, this corresponds to
reporting both the Brier mean squared error measure and $\AUC$.  See
\citet{Orallo2012} for a detailed, theoretically oriented comparison
of these and other performance measures under binary outcomes.

\section{Discussion}  \label{sec:discussion}

We have addressed a long-standing challenge in data analytics, by
introducing a set of tools --- comprising receiver operating
characteristic (ROC) movies, universal ROC (UROC) curves, and a
coefficient of predictive ability ($\CPA$) measure --- for generalized
ROC analysis, thereby freeing researchers from the need to
artificially binarize real-valued outcomes, which often is associated
with undesirable effects \citep{Altman2006}.  Throughout the paper, we
have assumed that predictors and features are linearly ordered,
thereby covering binary, ordinal, and continuous outcomes
simultaneously.  While our motivating example uses data from a
clinical trial, our approach does not account for censored data, as
typically encountered in survival analysis.  We strongly encourage
extensions of ROC movies, UROC curves and $\CPA$ that apply to
censored data, perhaps along the lines of \citet{Blanche2013}.  For
generalizations of ROC analysis to multi-class problems with
categorical outcomes that cannot be linearly ordered see
\citet{Hand2001}, \citet{Ferri2003}, and Section 9 of
\citet{Fawcett2006}.

ROC movies, UROC curves, and $\CPA$ reduce to the classical ROC curve
and $\AUC$ when applied to binary data.  Moreover, attractive
properties of ROC curves, such as invariance under strictly increasing
transformations and straightforward interpretability are maintained by
ROC movies and UROC curves.  In contrast to customarily used measures
of bivariate association and dependence \citep{Reshef2011, Weihs2018},
$\CPA$ is asymmetric, i.e., in general, its value changes if the roles
of the feature and the outcome are transposed.  However, when both the
feature and the outcome are continuous, $\CPA$ becomes symmetric, and
relates linearly to Spearman's rank correlation coefficient.  Thus,
$\CPA$ bridges and generalizes $\AUC$, Somers' $D$ and Spearman's rank
correlation coefficient, up to a linear relationship, just like the C
index connects and generalizes $\AUC$, Somers' $D$ and Kendall's rank
correlation coefficient.  While in typical practice the two measures
yield qualitatively similar results, under positive dependence $\CPA$
is larger than the C index, and $\CPA$ tends to be less affected by
discretization effects.

In view of the advent of dynamic graphics in mainstream scientific
publishing, we contend that ROC movies, UROC curves, and $\CPA$ are
bound to supersede traditional ROC curves and $\AUC$ in a wealth of
applications.  Open source code for their implementation in Python
\citep{python2021} and the \textsf{R} language and environment for
statistical computing \citep{rcore2021} is available on GitHub at
\url{https://github.com/evwalz/urocc} and
\url{https://github.com/evwalz/uroc}.

\section*{Acknowledgements}

We thank three anonymous referees, Zied Ben Bouall\`egue, Timo
Dimitriadis, Andreas Eberl, Dominic Edelmann, Andreas Fink, Alexander
I.~Jordan, Peter Knippertz, Sebastian Lerch, Marlon Maranan, Florian
Pappenberger, Johannes Resin, David Richardson, Peter Sanders, Johanna
F.~Ziegel, Philipp Zschenderlein and seminar participants at the
European Centre for Medium-Range Weather Forecasts (ECMWF) and
International Symposium on Forecasting (ISF) for advice, discussion
and encouragement.  In particular, Peter Knippertz provided detailed
comments on the WeatherBench example.  This work has been supported by
the Klaus Tschira Foundation, by the Helmholtz Association, and by the
Deutsche Forschungs\-ge\-mein\-schaft (DFG, German Research
Foundation) -- Project-ID 257899354 -- TRR 165.  Tilmann Gneiting
furthermore acknowledges travel support via the ECMWF Fellowship
programme.

\end{document}